\documentclass{article}

\usepackage{microtype}
\usepackage{graphicx}
\usepackage{subfigure}
\usepackage{booktabs} 

\usepackage{amsmath}
\usepackage{amssymb}
\usepackage{amsthm}
\usepackage{bm}
\usepackage{mathtools}
\usepackage{pifont}
\usepackage{subfigure}
\usepackage{todonotes}
\usepackage{stfloats}
\usepackage{adjustbox}
\usepackage{hyperref}

\newcommand*{\V}[1]{\mathbf{#1}}
\DeclarePairedDelimiter\norm{\lVert}{\rVert}
\newtheorem{theorem}{Theorem}[section]
\newtheorem{proposition}{Proposition}[section]
\DeclareMathOperator*{\argmin}{arg\,min}

\usepackage[accepted]{icml2021}

\icmltitlerunning{Heavy-tailed denoising score matching}

\begin{document}

\twocolumn[
\icmltitle{Heavy-tailed denoising score matching}



\icmlsetsymbol{equal}{*}

\begin{icmlauthorlist}
\icmlauthor{Jacob Deasy}{cam}
\icmlauthor{Nikola Simidjievski}{cam}
\icmlauthor{Pietro Li\`{o}}{cam}
\end{icmlauthorlist}
\icmlaffiliation{cam}{Department of Computer Science and Technology, University of Cambridge, Cambridge, United Kingdom}
\icmlcorrespondingauthor{Jacob Deasy}{jd645@cam.ac.uk}

\icmlkeywords{Machine Learning, Deep Learning, Generative Model, Score Based Model, ICML}

\vskip 0.3in
]



\printAffiliationsAndNotice{}  

\begin{abstract}
Score-based model research in the last few years has produced state of the art generative models by employing Gaussian denoising score-matching (DSM). However, the Gaussian noise assumption has several high-dimensional limitations, motivating a more concrete route toward even higher dimension PDF estimation in future. We outline this limitation, before extending the theory to a broader family of noising distributions---namely, the generalised normal distribution. To theoretically ground this, we relax a key assumption in (denoising) score matching theory, demonstrating that distributions which are differentiable \textit{almost everywhere} permit the same objective simplification as Gaussians. For noise vector norm distributions, we demonstrate favourable concentration of measure in the high-dimensional spaces prevalent in deep learning. In the process, we uncover a skewed noise vector norm distribution and develop an iterative noise scaling algorithm to consistently initialise the multiple levels of noise in annealed Langevin dynamics (LD). On the practical side, our use of heavy-tailed DSM leads to improved score estimation, controllable sampling convergence, and more balanced unconditional generative performance for imbalanced datasets.
\end{abstract}


\section{Introduction}\label{sec:Introduction}

Given a probability distribution $p(\V{x})$, $\V{x}\in\mathbb{R}^{n}$, the \textit{score function} is defined as
\begin{align}\label{eq:score}
    s(\V{x}) = \nabla_{\V{x}}\log p(\V{x}),
\end{align}
the gradient of the log-density with respect to the input $\V{x}$. The score is a vector field of the gradient at $\V{x}$, and gives the direction of the maximum increase in log-density.

\textit{Score based models} (SBMs) are parameterised and trained to estimate $\nabla_{\V{x}}\log p(\V{x})$. Unlike likelihood-based models, such as normalising flows \cite{rezende2015variational,kobyzev2020normalizing} or autoregressive models \cite{papamakarios2017masked}, this approach has the advantage of modelling an unconstrained function that does not need to be normalised.

By starting with the energy based model formulation
\begin{align}
    p_{\theta}(\V{x}) = e^{-f_{\theta}(\V{x})}/Z_{\theta},
\end{align}
for parameters $\theta\in\mathbb{R}^{m}$, with $m\gg1$ for deep learning models, it is clear that
\begin{equation}
  s_\theta (\V{x}) = \nabla_{\V{x}} \log p_\theta (\V{x} ) = -\nabla_{\V{x}}  f_\theta (\V{x}) - \underbrace{\nabla_\V{x} \log Z_\theta}_{=0} = -\nabla_\V{x} f_\theta(\V{x}),
\end{equation}
naturally removes the oft intractable \textit{partition function} $Z_{\theta}$.

The goal of SBMs is to fit $s_{\theta}(\V{x}):=\nabla_{\V{x}}\log p_{\theta}(\V{x})$ to $\nabla_{\V{x}}\log p_{\V{x}}(x)$, but of course $\nabla_{\V{x}}\log p_{\V{x}}(\V{x})$ is not available in the first place. As such, it is necessary to assess whether any given minimisation can avoid the tautologous use of $p(\V{x})$.

A simple first attempt to minimise the Euclidean distance, known as the Fisher divergence, across the space gives the explicit score matching (ESM) objective
\begin{align}
     \mathcal{J}_{ESMp}(\theta)= \frac{1}{2}\mathbb{E}_{p(\V{x})}\left[||\nabla_{\V{x}}\log p(\V{x}) - s_{\theta}(\V{x})||_{2}^{2}\right].\label{eq:esm}
\end{align}
Despite the continuing presence of $p(\V{x})$, a useful result is that, following an integration by parts, $\mathcal{J}_{ESMp}$ (ignoring a constant shift) simplifies to implicit score matching (ISM)
\begin{align}
    \mathcal{J}_{ISMp}(\theta) = \mathbb{E}_{p(\V{x})}\left[\frac{1}{2}\lVert s_{\theta}(\V{x})\rVert_{2}^{2} + \textrm{tr}\left(\nabla_{\V{x}}s_{\theta}(\V{x})\right)\right],\label{eq:ism}
\end{align}
where \textit{the density function of the observed data does not appear} \cite{hyvarinen2005estimation}. This integration is subject to a few weak constraints which are detailed in Section~\ref{sec:Proofs} as they motivate a theorem in Section~\ref{sec:High dimensional noising}.

In practice, discretising the expectation, $\mathcal{J}_{ISMp}$ is then approximated by
\begin{align}
    \mathcal{J}_{ISMp_{0}}(\theta) = \frac{1}{N}\sum\limits_{i=1}^{N}\left[\frac{1}{2}\norm[\bigg]{ s_{\theta}\left(\V{x}^{(i)}\right)}_{2}^{2} + \textrm{tr}\left(\nabla_{\V{x}}s_{\theta}\left(\V{x}^{(i)}\right)\right)\right],\label{eq:ism0}
\end{align}
for $N$ data samples, an intuitive objective where:
\begin{itemize}
    \item Term one minimises the scale of the score to zero, inducing the presence of a local minimum or maximum.
    \item Term two, the trace of the Jacobian of the score, being minimised then clearly indicates an objective forcing local maxima at each data point.
\end{itemize}


The trace of the Jacobian in \eqref{eq:ism} and \eqref{eq:ism0} requires $\mathcal{O}(n)$ backpropagations to calculate and is therefore computationally expensive enough to render this objective impractical. As an example of suggested optimisations, \citet{song2020sliced} proposed \textit{sliced score matching} (SSM) which projects the vectors onto random directions (far fewer than $n$ times) and takes the expectation of the objective over these directions. However, SSM has been superseded by a new form of \textit{denoising score matching} (DSM), originally from \cite{vincent2011connection}, which avoids the Jacobian altogether.

The first step of DSM is to perturb the data $\V{x}$ with a known noise distribution $q_{\sigma}(\Tilde{\V{x}}|\V{x})$ (normally convolution with a diagonal multivariate Gaussian kernel)
\begin{align}
    q_{\sigma}(\V{x})=\int_{\V{X}}q_{\sigma}(\Tilde{\V{x}}|\V{x})p_{\textrm{data}}(\V{x})\textrm{d}\V{x}.
\end{align}
The key step in \cite{vincent2011connection}, relying on the same assumptions in a similar integration by parts to that of \cite{hyvarinen2005estimation}, was to prove that \eqref{eq:ism} is equivalent to DSM
\begin{align}
    \mathcal{J}_{DSMq_{\sigma}}(\theta) = \frac{1}{2}\mathbb{E}_{q_{\sigma}(\Tilde{\V{x}}|\V{x})p(\V{x})}\left[\lVert s_{\theta}(\Tilde{\V{x}}) - \nabla_{\Tilde{\V{x}}}\log q_{\sigma}(\Tilde{\V{x}}|\V{x})\rVert_{2}^{2}\right],\label{eq:dsm}
\end{align}
with $s_{\theta^{*}}(\V{x})=\nabla_{\V{x}}\log q_{\sigma}(\V{x})$ almost surely, and $\nabla_{\V{x}}\log q_{\sigma}(\V{x})\approx\nabla_{\V{x}}\log p_{\textrm{data}}(\V{x})$ when the noise is low enough for $q_{\sigma}(\V{x})\approx p_{\textrm{data}}(\V{x})$. Crucially, perturbation of the distribution in \eqref{eq:dsm} is computationally trivial and only a single backpropagation is required.

A well-approximated score can then be used for sampling via Langevin Monte Carlo (LMC, \cite{besag1994comments}, summarised in Algorithm~\ref{alg:Langevin dynamics}). The major assumptions of note are the standard choice of Gaussian noise in LD, which can be replaced with heavier tailed noise sources \cite{csimcsekli2017fractional}, as well as the first-order nature of the iteration. In the last few years, the convergence rates of higher-order schemes of LD have been formalised \cite{cheng2018sharp,mou2019high} and their integration with score matching looks to be an interesting avenue of research beyond the scope of this work.

To further improve SBMs, \citet{song2019generative} suggested \textit{annealed Langevin dynamics} (ALD, see Algorithm~\ref{alg:Anneal langevin dynamics}). Follow-up work then made five training technique suggestions which allow improved scale and generation quality \cite{song2020improved}. The success of ALD has been such that recent SBMs are now on par, if not better (with heavy compute), than best-in-class GANs and autoregressive models \cite{song2020score,vahdat2021score}.

Due to the success of these improvements, research in this area has proliferated over the past two years. As a full review is beyond the scope of this work and yet to appear in the literature, a brief summary is included here. Critiques and expansions of both discrete and continuous (see Section~\ref{sec:Continuous extension to stochastic differential equations}) DSM have been presented in \cite{huang2021variational,kim2021score,song2021maximum}. DSM for discrete data was formally defined in \cite{hoogeboom2021argmax}, and techniques for sampling with score (and higher order \cite{meng2021estimating}) estimates have experienced a renaissance \cite{jolicoeur2021gotta}. Additionally, the connection between SBMs and denoising diffusion probabilistic models (DDPMs) has been clarified \cite{ho2020denoising,song2020score}. Finally, closely related to our work, is the first use of non-Gaussian noise in DDPMs in \citet{nachmani2021non}, to assess the effects of noise with more degrees of freedom.

This paper\footnote{Code is available at \href{https://github.com/jacobdeasy/heavy-tail-dsm}{\color{blue}github.com/jacobdeasy/heavy-tail-dsm}.} builds on this with the following \emph{contributions}:
\begin{itemize}
    \item Insight into the undesirable $n$-dimensional annuli in Gaussian DSM and novel theoretical expansion of DSM to heavy-tailed DSM.
    \item Introduction of the generalised normal noise family to DSM for controllable diffusion strength.
    \item Image generation results across a range of datasets and metrics which are both competitive with standard DSM and reduce class imbalance.
    \item An initial description of how L\'{e}vy-flight-like sampling paths can be used by continuous SBMs.
\end{itemize}

\begin{figure}[h]
    {\centering
    \subfigure[$p(\V{x})$]{\label{fig:dsm_1}\includegraphics[width=0.5\linewidth]{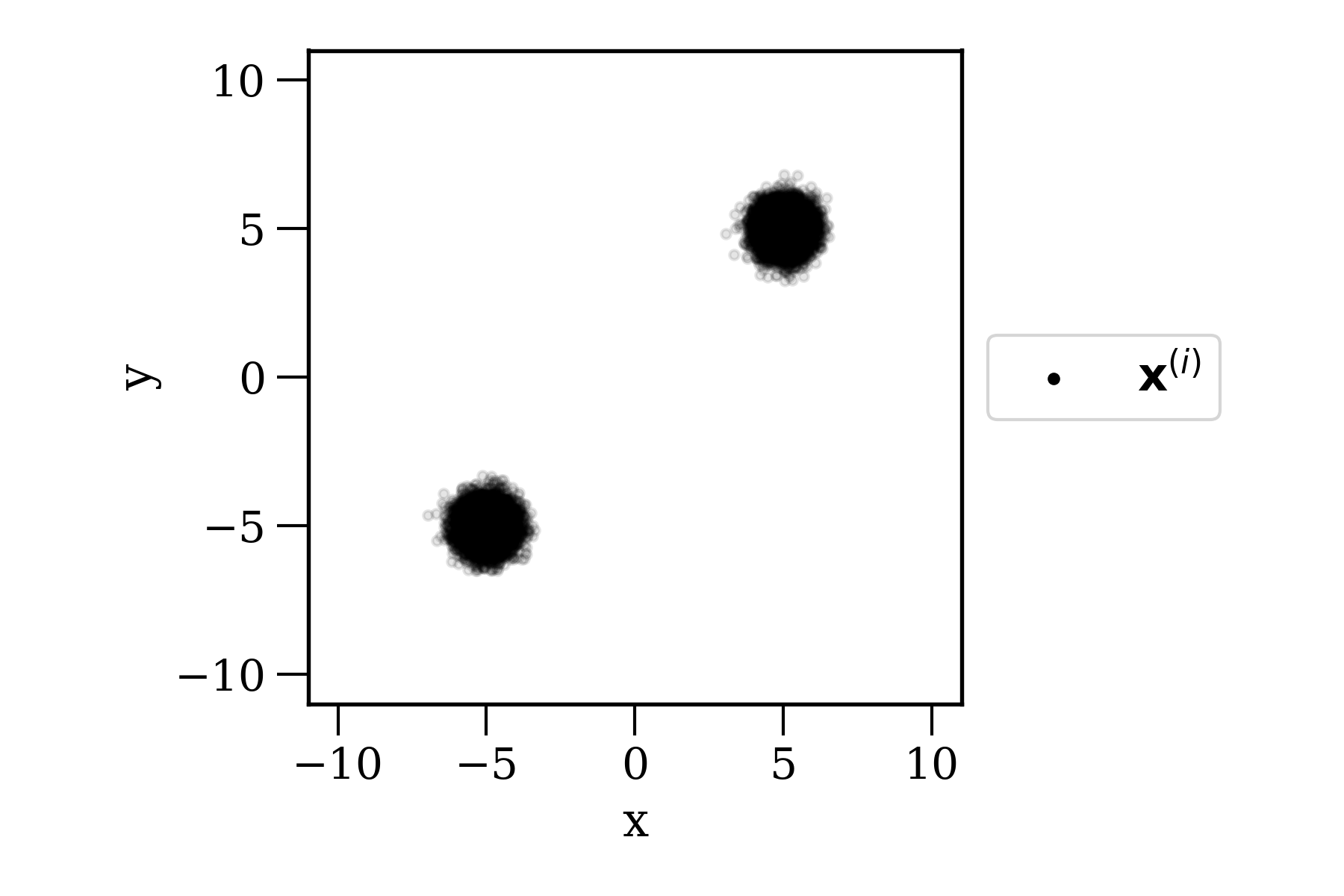}}%
    \subfigure[Sample paths.]{\label{fig:dsm_6}\includegraphics[width=0.5\linewidth]{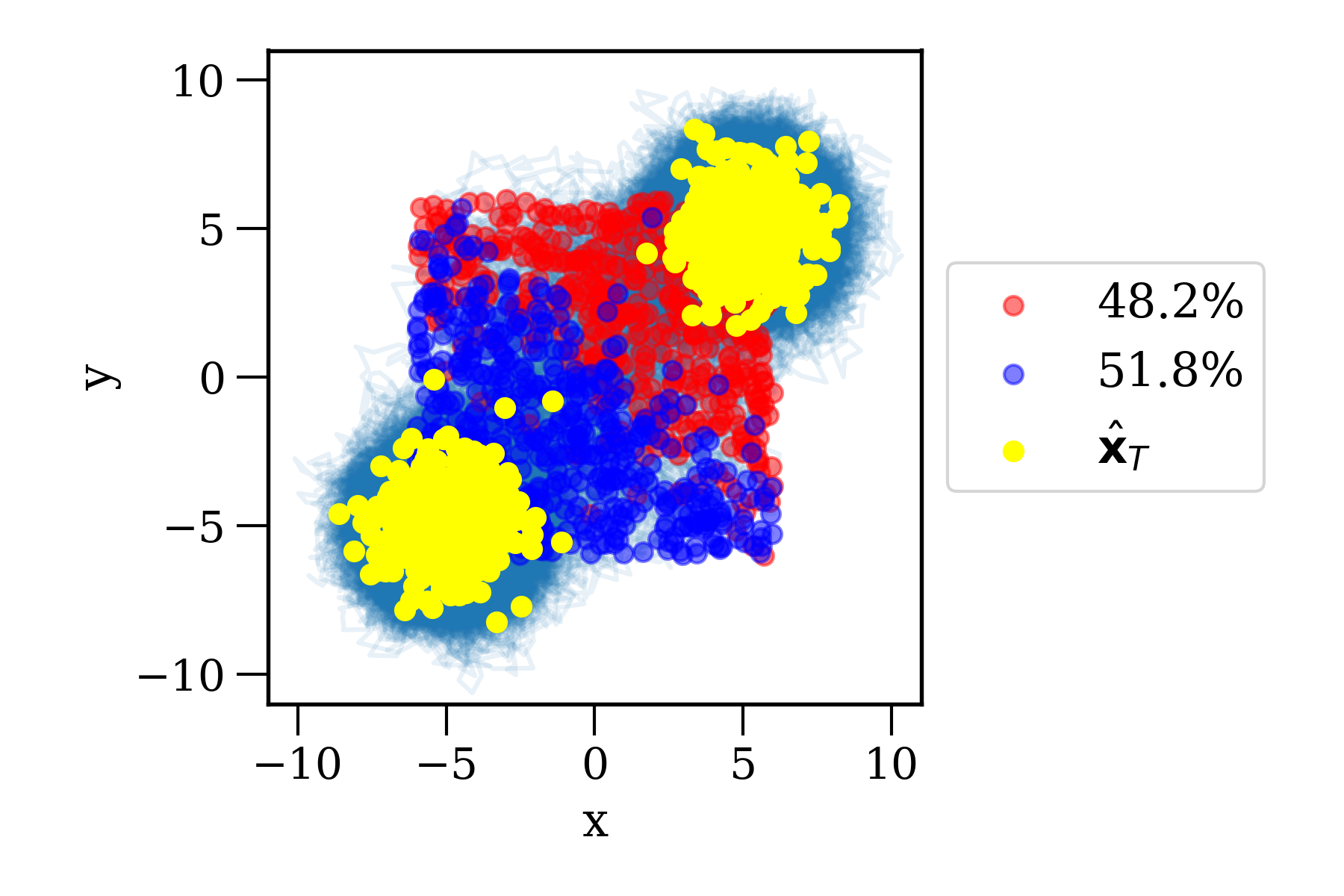}}}
    \caption{DSM training and LD sampling. In \textbf{a}, $p(\V{x})$ is modelled as an additive mixture of ($k=2$) bivariate Gaussians with 20,000 samples. An MLP is trained to estimate the score from samples noised by $q_{\sigma}(\Tilde{\V{x}}|\V{x})\sim\mathcal{N}(\V{x},\V{I})$. 1,000 sampled paths are evolved in \textbf{b} to demonstrate the decision boundary, its asymmetry (relevant for class imbalance), and the upper bound on approximation accuracy due to the underlying unit noise. Full details in Figure~\ref{fig:dsm example 2}.}
    \label{fig:dsm example 1}
\end{figure}

\begin{figure}[h]
    {\centering
    \subfigure[]{\label{fig:dsm_ald_1}\includegraphics[width=0.5\linewidth]{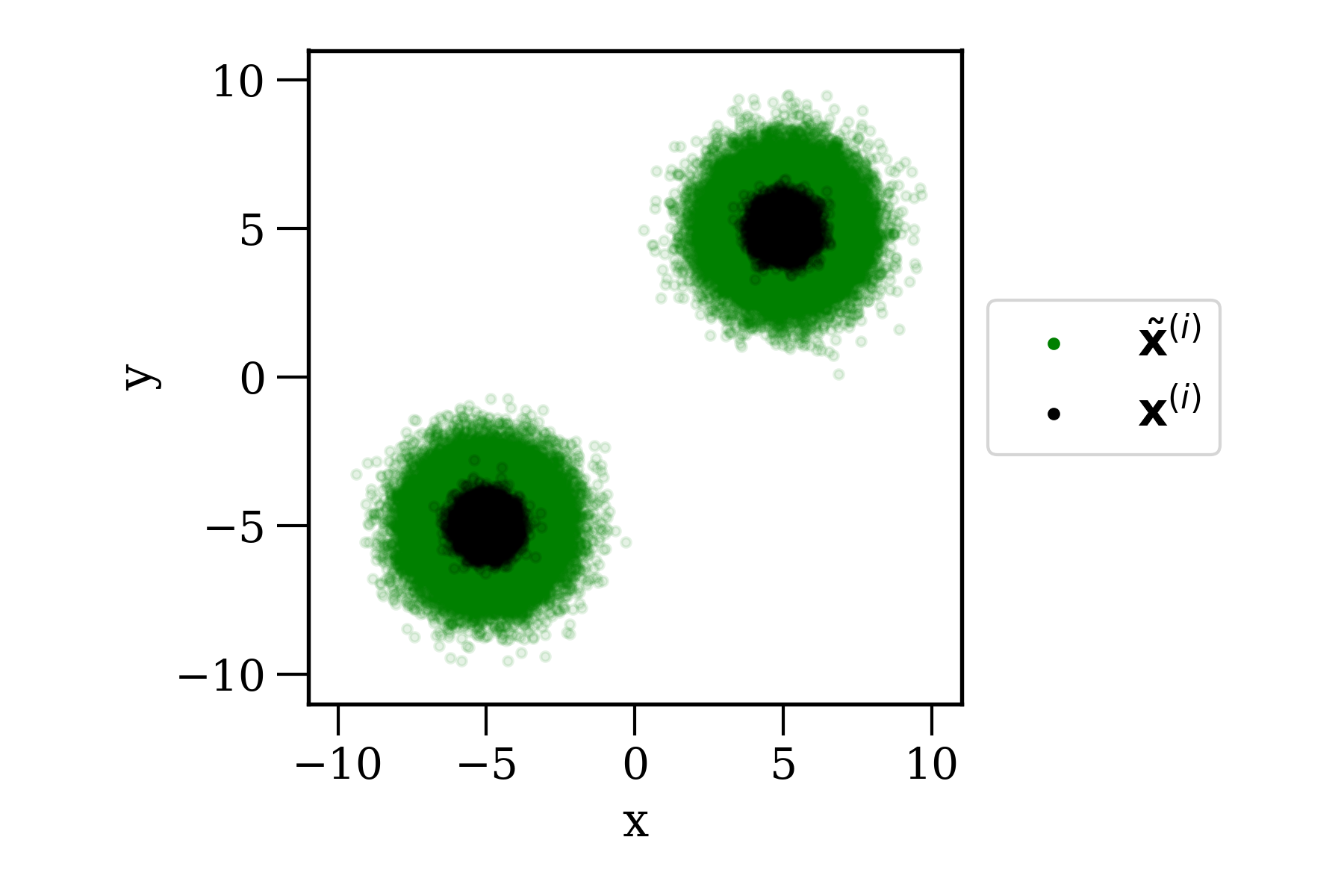}}%
    \subfigure[]{\label{fig:dsm_ald_4}\includegraphics[width=0.5\linewidth]{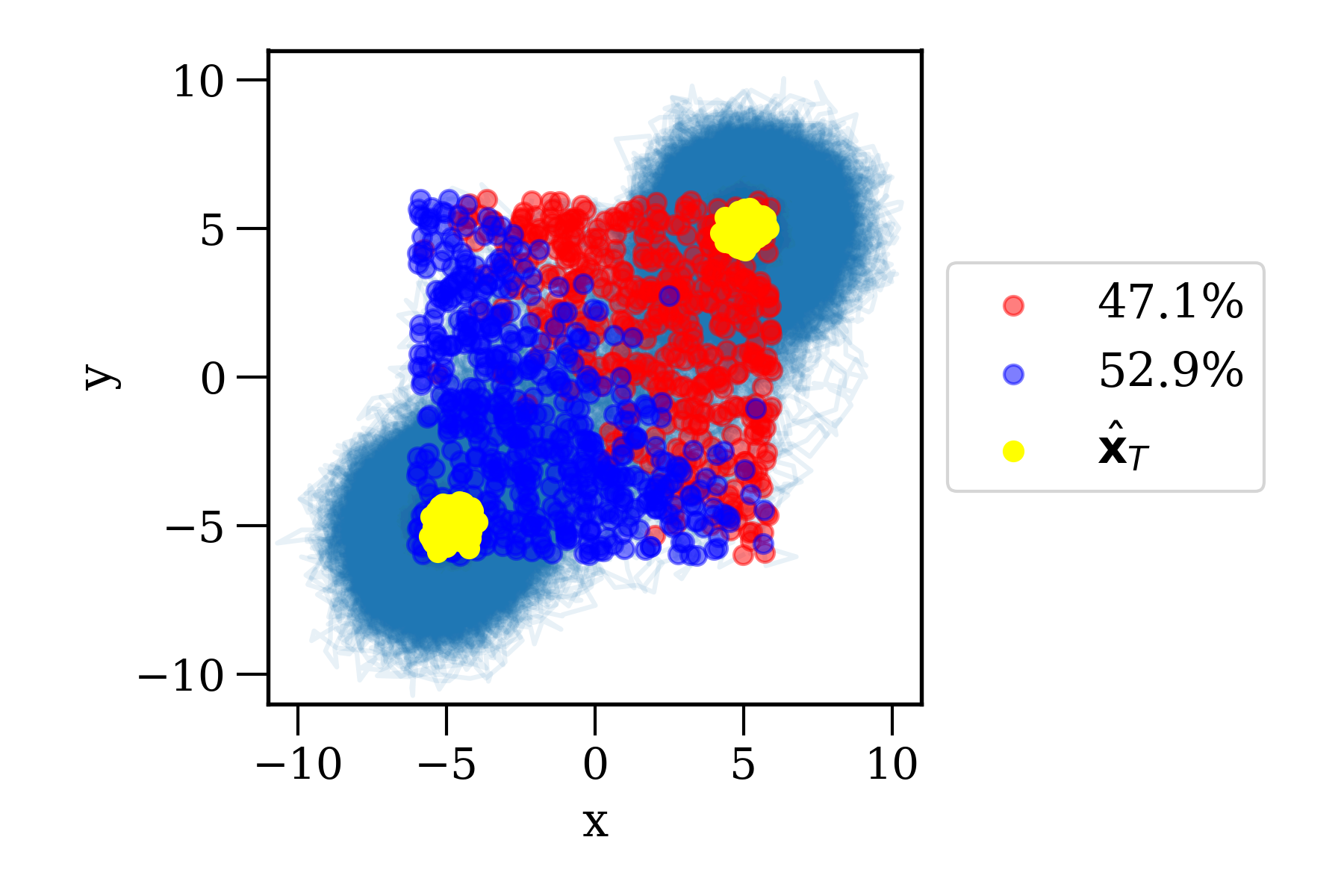}}}
    \caption{Multiple noise level DSM training and ALD sampling. The setup and figures are identical to Figure~\ref{fig:dsm example 1} except that two noise scales, $\sigma_{1}=1.0$ and $\sigma_{2}=0.25$, are used. Full details in Figure~\ref{fig:dsm ald example 2}.}
    \label{fig:dsm ald example 1}
\end{figure}

\section{High dimensional noising}\label{sec:High dimensional noising}

The previous section discussed the background of denoising score matching and finished with various strategies to scale this process to higher dimensions and better sample quality. This section will consider the weaknesses of those scaling strategies, with a particular focus on generalising high dimensional noise perturbations.


\subsection{Beyond Gaussian noise}

To elucidate how Gaussian DSM works in practice, in Figure~\ref{fig:dsm example 1}, a $2D$ example is provided, demonstrating the method converging and generating samples from a mixture of Gaussians. The example considers all steps of the procedure: noising, training, and sampling. Then, in Figure~\ref{fig:dsm ald example 1}, the example is extended to the multiple noise levels in DSM with ALD. Although the improvement between the two is evident, this synthetic example will be used to highlight the weaknesses of \textit{Gaussian} DSM with ALD in Section~\ref{sec:low dimension dsm results}.

In \citet{vincent2011connection}, the choice of Gaussian noise is for convenience and has the bonus of an intuitive score
\begin{align}
    \nabla_{\Tilde{\V{x}}}\log q_{\sigma}(\Tilde{\V{x}}|\V{x})=\Sigma^{-1}(\V{x}-\Tilde{\V{x}})\label{eq:gaussian_score},
\end{align}
which corresponds to moving from noisy $\Tilde{\V{x}}$ to clean $\V{x}$. By considering the actual constraint on the noise distribution, that $\log q_{\sigma}(\Tilde{\V{x}}|\V{x})$ is differentiable \cite{vincent2011connection}, this subsection will explore the consequences of the pivotal Gaussian noise assumption in $\mathcal{J}_{DSMq_{\sigma}}$. The differentiability condition encompasses a broad range of potential distributions and gives rise to the questions: \textit{What form could and should $q_{\sigma}$ take? How does the choice of $q_{\sigma}$ influence model learning?}

\paragraph{Gaussian noise in high dimensions.}\label{sec:Gaussian noise in high dimensions}

Now that intuition about the role of the noising process in DSM has been established, it is necessary to consider the effects of noising in higher dimensions.

For instance, consider how the \textit{squared} $L^{2}$ \textit{norm distribution} of the isotropic Gaussian vector
\begin{align}
    Y = \lVert\mathbf{X}\rVert_{2}^{2},
\end{align}
follows either of the chi-squared distributions
\begin{align}
    Y\sim \chi^{2}(n) = n\chi^{2}(1),
\end{align}
which approaches a Gaussian distribution centred at $n$ in the limit $n\to\infty$. See Section~\ref{sec:Unintuitive high-dimensional statistics} for the standard derivation, Section~\ref{sec:Comparison of concentration moments} for a full description of the moments, and Figure~\ref{fig:chi-squared} for a visualisation of the chi-squared distribution for increasing degrees of freedom.

\paragraph{The problem in high-dimensional SBMs.}
The conditional DSM Gaussian noise distribution $q_{\sigma}$ smooths around each data point $\V{x}$. In the ideal scenario, the surrounding Gaussian $n$-spheres would overlap slightly, filling the high-dimensional convex hull defined by the dataset. As such, throughout the hull, the SBM would learn to faithfully interpolate the space, estimating gradients accurately so that they can be used in an iterative generation procedure. However, using standard results, it is clear that in the high-dimensional setting of deep learning, these $n$-spheres in fact approach $n$-annuli---very thin shells.

This consideration immediately offers a new interpretation of why multiple levels of noise in ALD were a major improvement over prior methods. Although the original motivation in \cite{song2019generative} was to enable LD noise annealing, similar to annealed importance sampling \cite{neal2001annealed}, this step also stacked concentric noise annuli. Therefore, SBMs with ALD learn gradients that apply to a larger volume of the dataset interior. Moreover, for the original ALD paper, this perspective confounds the performance improvement due to annealing the Langevin dynamics with `filling the convex hull'. Such an insight motivates decoupling of, and clarification around, the effect of both approaches.

Despite the follow-up improvements to ALD \cite{song2020improved}, recognising this concentration of noise and increasing the number of noise levels, the authors' motivation was to correctly balance coverage across regions of different weight. This interpretation can be taken further and permits several opportunities:
\begin{itemize}
    \item Even with multiple levels of noise, how do these models fair when generating sparse distributions---what is the performance-sparsity trade-off? As highlighted in Figure~\ref{fig:heavy_tailed_distributions}, the Gaussian distribution has relatively light tails compared to several reasonably well-behaved distributions that have been studied in-depth. Heavier tails should facilitate sampling further across sparse domains and aid score interpolation.
    \item At the time of writing, noise level selection for the best discrete ALD model is sampled linearly in log space between two hyperparameters for the minimum and maximum noise level. In the continuous case, recent models have tried to learn this distribution \cite{kingma2021variational}, but the resulting approximation has not been theoretically explained. Clearly, refinement of the discrete case, potentially leading to an explanation in the continuous case, is a motivating theoretical goal. Moreover, the relationships between data dimension, DSM noise distribution, and DSM noise-vector norm distribution have not been explored. In particular, Section~\ref{sec:Choosing a scale parameter sequence for arbitrary noise} addresses skewed norm distributions and general noise with a quantile matching algorithm. 
\end{itemize}

As a result, it is tempting to turn toward common heavy tailed distributions, such as those in Figure~\ref{fig:heavy_tailed_distributions}. Unfortunately, as a first port of call, the Cauchy distribution is notoriously difficult to manipulate, evidenced by its undefined moments, and does not permit the same concentration analysis as in Section~\ref{sec:Gaussian noise in high dimensions} \cite{eicker1985sums}. The same issue, the \textit{summation} of the squared RV rather than the squaring itself, arises for the Student-$t$ distribution because the square of a $t$-distribution is an $F$-distribution \cite{box2011bayesian} which has an undefined MGF.

Nevertheless, the square of a Laplace random variable follows a Weibull distribution and, as the Weibull distribution is linear in its first parameter, the squared norm distribution is a Weibull distribution also. Notwithstanding the potential of this result, it is possible to go further by considering a much broader family of distributions that both subsumes the Gaussian and Laplace distributions and is similarly equipped with a tractable PDF.

\paragraph{Generalising to the generalised normal (exponential power) distribution.}
The generalised normal (GN) distribution \cite{nadarajah2005generalized}, $X_{i}\sim\mathcal{GN}(\mu,\alpha,\beta)$ with $\mu\in\mathbb{R}$ and $\alpha,\beta\in\mathbb{R}_{+}$, has PDF
\begin{align}
    f_{X}(x;\mu,\alpha,\beta) = \frac{\beta}{2\alpha\Gamma(1/\beta)}\exp\left\{-\left(\frac{|x-\mu|}{\alpha}\right)^{\beta}\right\},
\end{align}
which recovers the standard Gaussian distribution for $(\alpha,\beta)=\left(\sqrt{2},2\right)$, the standard Laplace distribution for $(\alpha,\beta)=(1,1)$, and the uniform density for $\beta=0$. The GN is used when the concentration of values around the mean and the tail behaviour are of particular interest \cite{box2011bayesian}, apt for this case.
\begin{figure}[t]
    \centering
    \subfigure[Log PDF.]{\includegraphics[width=0.5\linewidth]{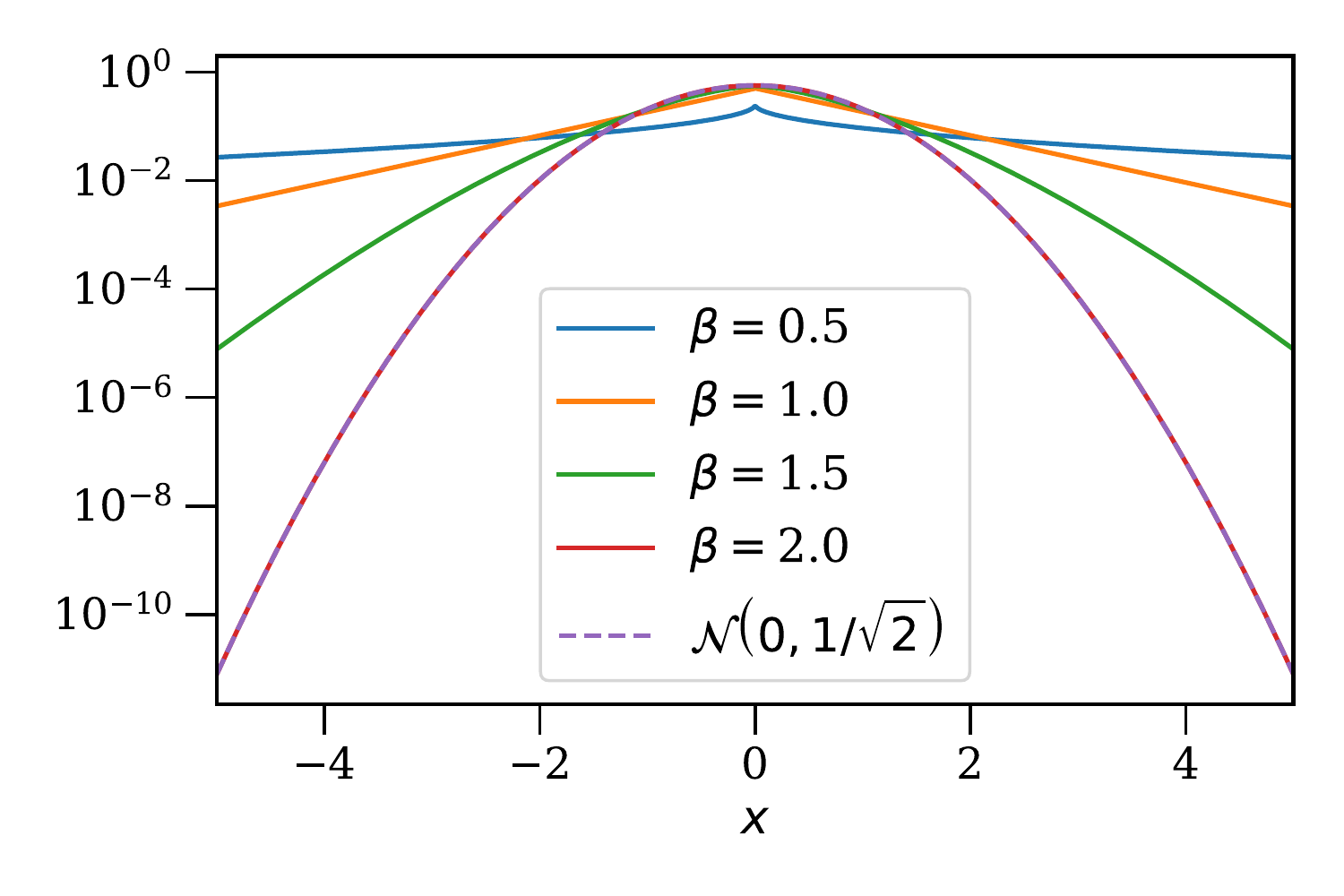}}%
    \subfigure[Score function.]{\label{fig:gn score}\includegraphics[width=0.5\linewidth]{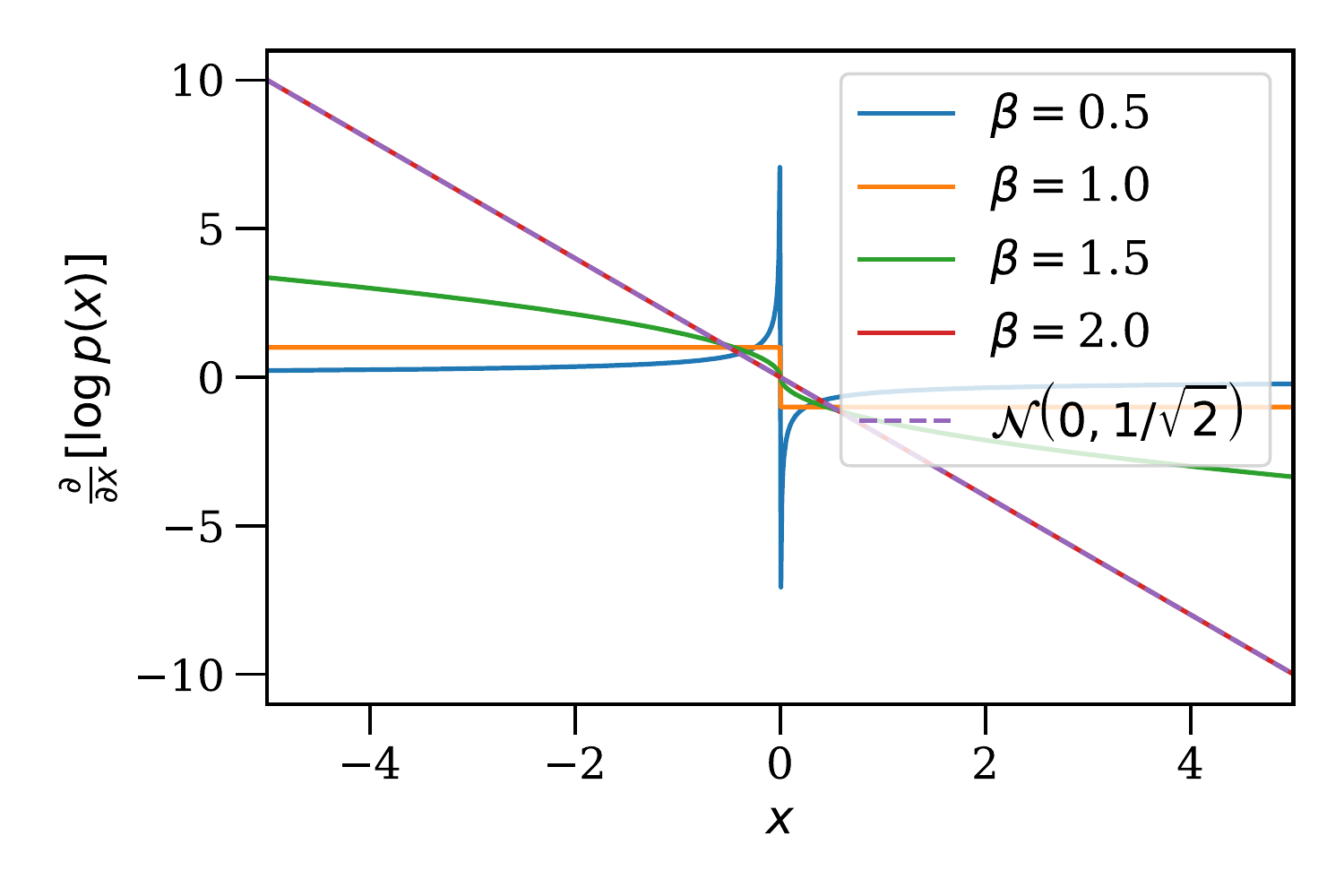}}
    \caption{The generalised normal distribution for varied $\beta$.}
    \label{fig:generalised normal}
\end{figure}

The corresponding score of $1D$ GN noise is
\begin{align}
    \frac{d}{d\Tilde{x}}\left[\log q_{\mathcal{GN}}(\Tilde{x}|x)\right] &= \frac{d}{d\Tilde{x}}\left[- \left(\frac{|\Tilde{x}-x|}{\alpha}\right)^{\beta}\right]\\
    &= -\frac{\beta}{\alpha^{\beta}}\textrm{sign}(\Tilde{x}-x)|\Tilde{x}-x|^{\beta-1}\label{eq:gn_score},
\end{align}
clearly indicating that the score of the generalised normal distribution is continuous but not differentiable at zero. This contravenes the necessary assumptions for DSM from \cite{vincent2011connection}. Therefore, to use this more general family of generalised normal distribution noise, it is necessary to weaken the theoretical constraints to piecewise-differentiable distributions.

\begin{theorem}\label{thm:piecewise differentiable}
Assume that the estimated score function $s_{\theta}(\V{x})$ obeys the assumptions outlined in \cite{vincent2011connection}, except that $s_{\theta}(\V{x})$ is instead differentiable \underline{almost everywhere}. Then, the objective function for $\mathcal{J}_{ESMp}$ in \eqref{eq:esm} is still equivalent to $\mathcal{J}_{ISMp}$ in \eqref{eq:ism}. Proof in Appendix~\ref{sec:Proofs}.
\end{theorem}
This theorem establishes that the generalised normal distribution, or similar distributions such as the Laplace distribution, can be used for noising in DSM. Therefore, to motivate its usage, further theoretical results are now derived surrounding GN concentration of measure.

\paragraph{Comparison of concentration moments.}\label{sec:Comparison of concentration moments}

Similar to the Gaussian case, considering the squared $L^{2}$ norm distribution $Y$, but deriving the distribution of $Z_{i}=X_{i}^{2}$ first for simplicity
\begin{align}
    F_{Z_{i}}(z) = P(Z_{i}\leq z) = P\left(X_{i}^{2}\leq z\right) = P\left(|X_{i}|\leq \sqrt{z}\right),
\end{align}
and therefore
\begin{align}
    f_{Z_{i}}(z) = F_{Z_{i}}'(z) &= \frac{1}{\sqrt{z}}\phi_{\mathcal{GN}}\left(\sqrt{z}\right)\\
    &=\frac{1}{\sqrt{z}}\frac{\beta}{2\Gamma(1/\beta)}\exp\left\{-|\sqrt{z}|^{\beta}\right\},
\end{align}
where $\phi_{\mathcal{GN}}$ is the corresponding unit distribution $(\alpha,\mu)=(1,0)$ and, as the square of the domain of $X_{i}$ is $\mathbb{R}_{+}$, the modulus can be ignored.

Recalling the \textit{generalised Gamma distribution} \cite{stacy1962generalization}, $\mathcal{GG}(a,d,p)$ with PDF
\begin{align}
    f(x;a,d,p)=\frac{p/a^{d}}{\Gamma(d/p)}x^{d-1}\exp\left\{-\left(\frac{x}{a}\right)^{p}\right\},
\end{align}
with $d\in\mathbb{R}$ and $a,p\in\mathbb{R}_{+}$, it becomes clear that $Z_{i}\sim\mathcal{GG}(a=1,d=1/2,p=\beta/2)$ and the sum is resolved by
\begin{align}
    f_{Y}(y) &= \frac{1}{n}\phi_{\mathcal{GG}}\left(\frac{y}{n}\right)\\
    &= \frac{1}{n}\frac{p}{a^{d}\Gamma(d/p)}\left(\frac{y}{n}\right)^{d-1}\exp\left\{-\left(\frac{y}{an}\right)^{p}\right\}\\
    &= \frac{p}{(an)^{d}\Gamma(d/p)}y^{d-1}\exp\left\{-\left(\frac{y}{an}\right)^{p}\right\},
\end{align}
showing $n\mathcal{GG}(a,d,p)=\mathcal{GG}(na,d,p)$, which means
\begin{align}
    Y\sim\mathcal{GG}(a=n,d=1/2,p=\beta/2). \label{eq:generalised nomal length distribution}
\end{align}
This derivation establishes the norm distribution appropriate for the generalised normal noise vector. To determine whether such a generalisation is useful, it is relevant to analyse how the moments of this distribution evolve with $n$.

For a Gaussian noise vector, the moments of $\lVert\mathbf{X}\rVert_{2}^{2}=Y\sim~\chi^{2}(n)$ are:
\begin{align}
    \mathbb{E}[Y] &= n \\
    \textrm{Var}(Y) &= 2n \label{eq:correct chi2 var}\\
    \textrm{Skew}(Y) &= \sqrt{\frac{8}{n}}\xrightarrow[]{n\to\infty} 0 \label{eq:chi2 skew}\\
    \textrm{Kurtosis}(Y) &= \frac{12}{n}\xrightarrow[]{n\to\infty} 0.
\end{align}

The main properties of interest here are:
\begin{itemize}
    \item Variance scales linearly with dimension. This moment dictates how thick the concentric annuli are for the sequence of noise levels in DSM with ALD.
    \item Skew and kurtosis tend to zero as dimensionality increases, leading to a near-Gaussian distribution in the limit. Also, when the skew value is non-zero, any scheme to overlap concentric annuli will be consistently biased toward one side of each annulus.
\end{itemize}
On the other hand, for the generalised Gaussian noise vector, the first two moments of $\lVert\mathbf{X}\rVert_{2}^{2}=Y\sim~\mathcal{GG}(a=n,d=1/2,p=\beta/2)$ are:
\begin{align}
    \mathbb{E}[Y] &= a\frac{\Gamma((d+1)/p)}{\Gamma(d/p)} = nC_{1} \\
    \textrm{Var}(Y) &= a^{2}\left(\frac{\Gamma((d+2)/p)}{\Gamma(d/p)} - \left(\frac{\Gamma((d+1)/p)}{\Gamma(d/p)}\right)^{2}\right) = n^{2}C_{2},
\end{align}
where
\begin{align}
    C_{1} &:= \frac{\Gamma(3/\beta)}{\Gamma(1/\beta)}\label{eq:mean_scaling} \\
    C_{2} &:= \frac{\Gamma(5/\beta)}{\Gamma(1/\beta)} - \left(\frac{\Gamma(3/\beta)}{\Gamma(1/\beta)}\right)^{2}. \label{eq:var_scaling}
\end{align}
Relative to the Gaussian moments, it should be noted that \eqref{eq:mean_scaling} scales the mean in a nonlinear fashion. As depicted in Figure~\ref{fig:gamma_trends2}, \eqref{eq:mean_scaling} undergoes a super-exponential decay with respect to $\beta$. Therefore, use of low $\beta$ noising strategies will push annulus samples relatively far from the base data point compared to Gaussian noise. It is apparent that a trade-off has emerged, between the desire for heavy-tails in DSM to fill high-dimensional space, and the unwieldy resulting norm distributions. This is perhaps, unsurprising, given the renowned difficulties when working with L\'{e}vy-like distributions \cite{mandelbrot1982fractal}. Secondly, the property that variance scales quadratically with dimension has surfaced. Therefore, as long as \eqref{eq:var_scaling} is greater than $1/n$---almost a guarantee given that $n\gg1$ and the even more aggressive exponential for low $\beta$ in Figure~\ref{fig:gamma_trends3}---substantially thicker shells will be present\footnote{Despite the unintuitive form of the gamma functions comprising \eqref{eq:mean_scaling} and \eqref{eq:var_scaling}, both terms simplify for $\beta\in\{2,1,\ldots,1/k\},\ k\in\mathbb{N}$, to Gaussian, Laplace, and closed-form moments respectively.}. Of course, this gain comes with the caveat that low $\beta$ noise is likely to be problematic, corresponding to score functions with a singularity at zero (evidenced in Figure~\ref{fig:generalised normal}).

\section{Results}\label{sec:Chapter 4 - Results}

After the groundwork of the previous section, this section designs empirical experiments to explore and confirm the utility of heavy-tailed denoising score matching (HTDSM). A qualitative and quantitative assessment of the insights of Section~\ref{sec:High dimensional noising} is provided at multiple scales, each lending support to the use of HTDSM in practice.

\subsection{Low dimensional space}\label{sec:low dimension dsm results}

Before progressing to high-dimension DL image datasets, it is apt to begin with an easily controlled and visualised continuation of the $2D$ example given in Figures~\ref{fig:dsm example 1} and \ref{fig:dsm ald example 1}.\footnote{Details of the implementation are provided in Section~\ref{sec:Training}.}

As a first implementation of the HTDSM scheme described in Section~\ref{sec:High dimensional noising}, Figure~\ref{fig:dsm ald laplace 1} (expanded in Figure~\ref{fig:dsm ald laplace 2} in Appendix~\ref{sec:2d_rez}) combats the density approximation task of Figure~\ref{fig:dsm ald example 1} using Laplace ($\beta=1$) noise. Figure~\ref{fig:dsm_ald_laplace_1} illustrates the diamond, rather than circular, noise structure of a diagonal bivariate Laplace distribution. Figure~\ref{fig:dsm_ald_laplace_2} and \ref{fig:dsm_ald_laplace_3} respectively demonstrate that ALD training \textit{and sampling} converge with Laplace (sub-Gaussian, piece-wise differentiable) noise, confirming Theorem~\ref{thm:piecewise differentiable}. The effect of the heavier-tailed noise when sampling is evidently present for the first half of the noise levels in Figure~\ref{fig:dsm_ald_laplace_4}, but this effect is outweighed by the down-scaling of ALD in the second half of sampling. Paths in Figure~\ref{fig:dsm_ald_laplace_5} begin from any point in the initialisation, extend across a far broader space, and all converge. The use of sub-Gaussian sampling diffusion is a novel step beyond standard ALD using SBMs and is closely aligned with fractional Langevin Monte Carlo methods \cite{csimcsekli2017fractional}. A positive is the removal of any kind of decision boundary, but a negative is the slightly inaccurate final solution. One way to solve this inaccuracy would be to simply add another, lower noise, level of ALD.

Figure~\ref{fig:dsm ald laplace 1} also clarifies that the higher variance in shell radii, derived in Section~\ref{sec:Comparison of concentration moments}, is in fact the variance arising due to the non-spherical nature of the high-dimensional generalised normal distribution. For instance, in the case $\beta=1$, Laplace noise provides samples in an approximate hypercube around its centre. The corners of the hypercube extend further along the axes than the Gaussian case, sacrificing probability mass not aligned with the coordinate system. It is noteworthy that this hypercube is approximate and the infinite domain of the Laplace distribution is therefore still more useful than the fixed hypercube of the uniform distribution. Immediate extensions are available, such as using a radial basis for the noise distribution, similar to that used in \citet{farquhar2020radial} for Bayesian neural network parameterisation. However, this direction is beyond the scope of this work and the Cartesian basis will continue to be used throughout.

\begin{figure}
    {\centering
    \subfigure[]{\label{fig:dsm_ald_laplace_1_1}\includegraphics[width=0.5\linewidth]{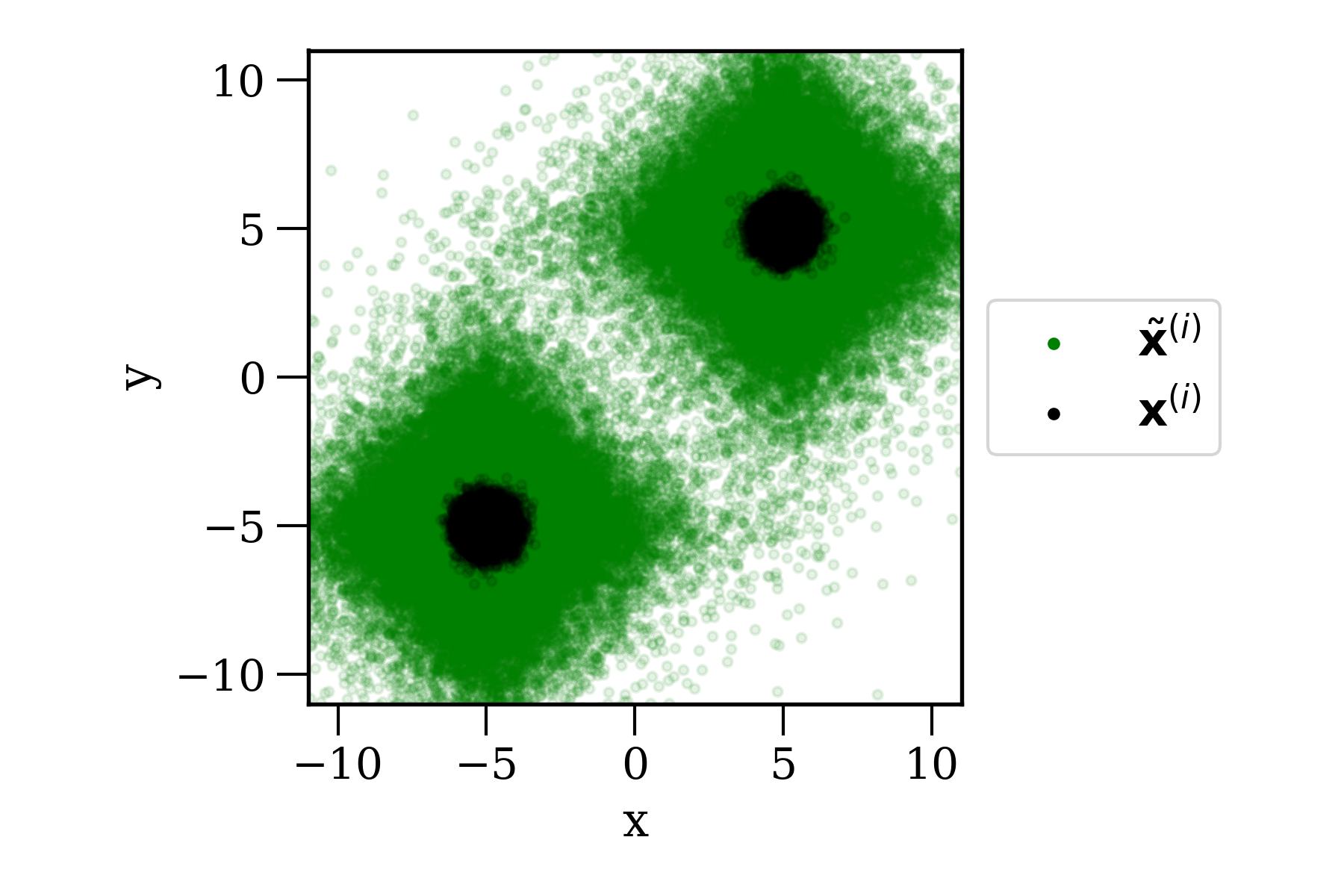}}%
    \subfigure[]{\label{fig:dsm_ald_laplace_1_3}\includegraphics[width=0.5\linewidth]{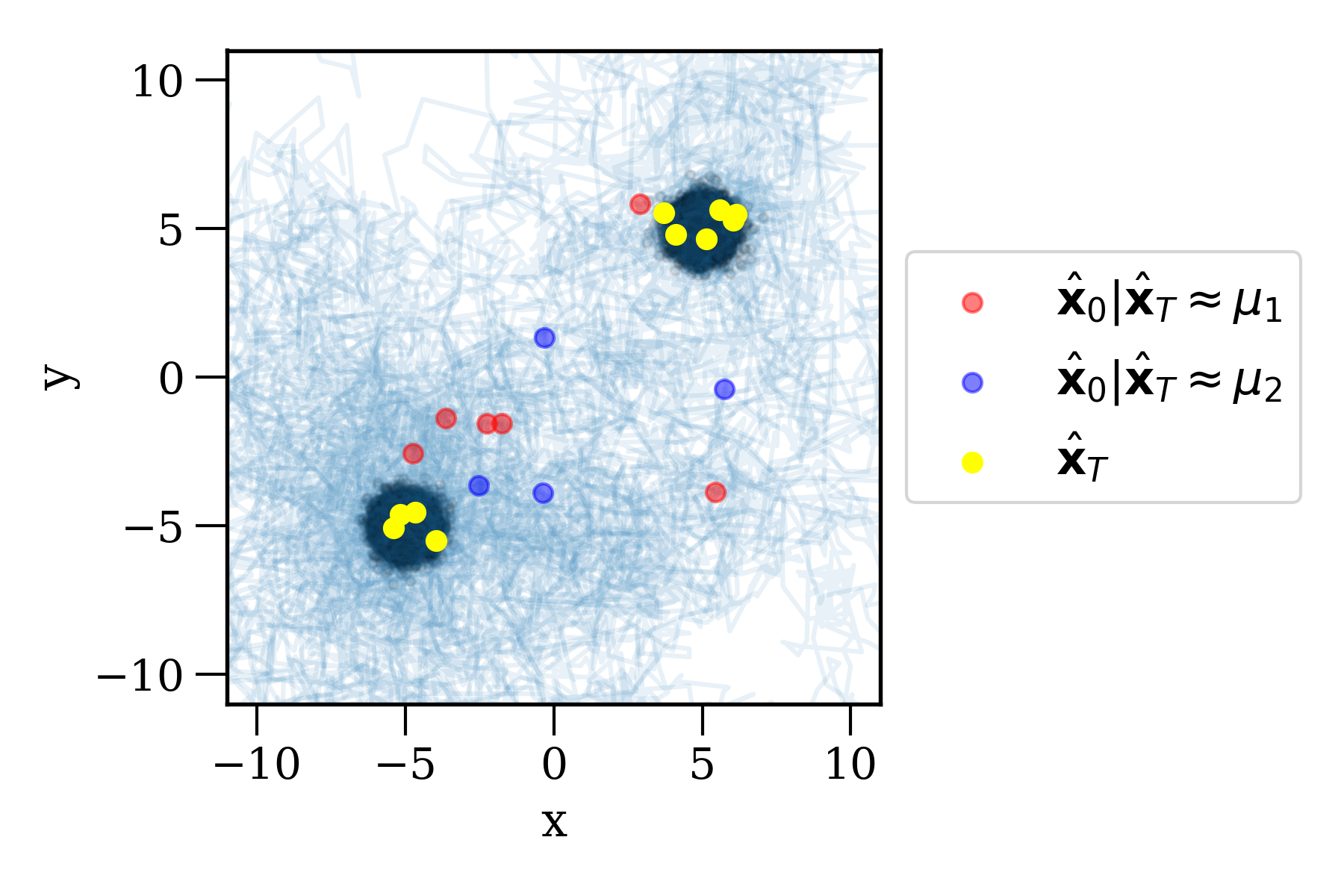}}}%
    \caption{Laplace DSM with ALD. The setup and figures are identical to Figure~\ref{fig:dsm ald example 1}, except that Laplace noise is used ($\beta=1$ in the general formulation). \textbf{a} depicts the diamond, rather than circular, noise structure of a diagonal bivariate Laplace distribution. \textbf{b} demonstrates that ALD sampling converges even with Laplace (sub-Gaussian, piece-wise differentiable) diffusion, confirming Theorem~\ref{thm:piecewise differentiable}. Full details in Figure~\ref{fig:dsm ald laplace 2}.}
    \label{fig:dsm ald laplace 1}
\end{figure}


Figure~\ref{fig:dsm ald 10x example 2}, in Appendix~\ref{sec:2d_rez}, also depicts how standard DSM with ALD can suffer from mode collapse. The setup and subfigures are identical to Figure~\ref{fig:dsm ald example 1}, except that $p(\V{x})$ samples now have an imbalance of 10:1 between modes 1 (upper right) and 2 (lower left) respectively. Particle paths in Figure~\ref{fig:dsm_ald_10x_2} and Figure~\ref{fig:dsm_ald_10x_4} clearly show a preference for mode 1, even crossing mode 2 entirely. This is arguably not a problem, but 97.2\% of particles approaching mode 1 in Figure~\ref{fig:dsm_ald_10x_4} does not reflect the true imbalance ($>99\%$ is also not an uncommon steady-state for this setup). The large scores associated with distant particle migration, across mode 2 to mode 1, can also be seen by the scale of the initial scores in Figure~\ref{fig:dsm_ald_10x_3}.

\begin{figure}[t]
    {\centering
    \subfigure[]{\label{fig:dsm_ald_laplace_10x_1}\includegraphics[width=0.5\linewidth]{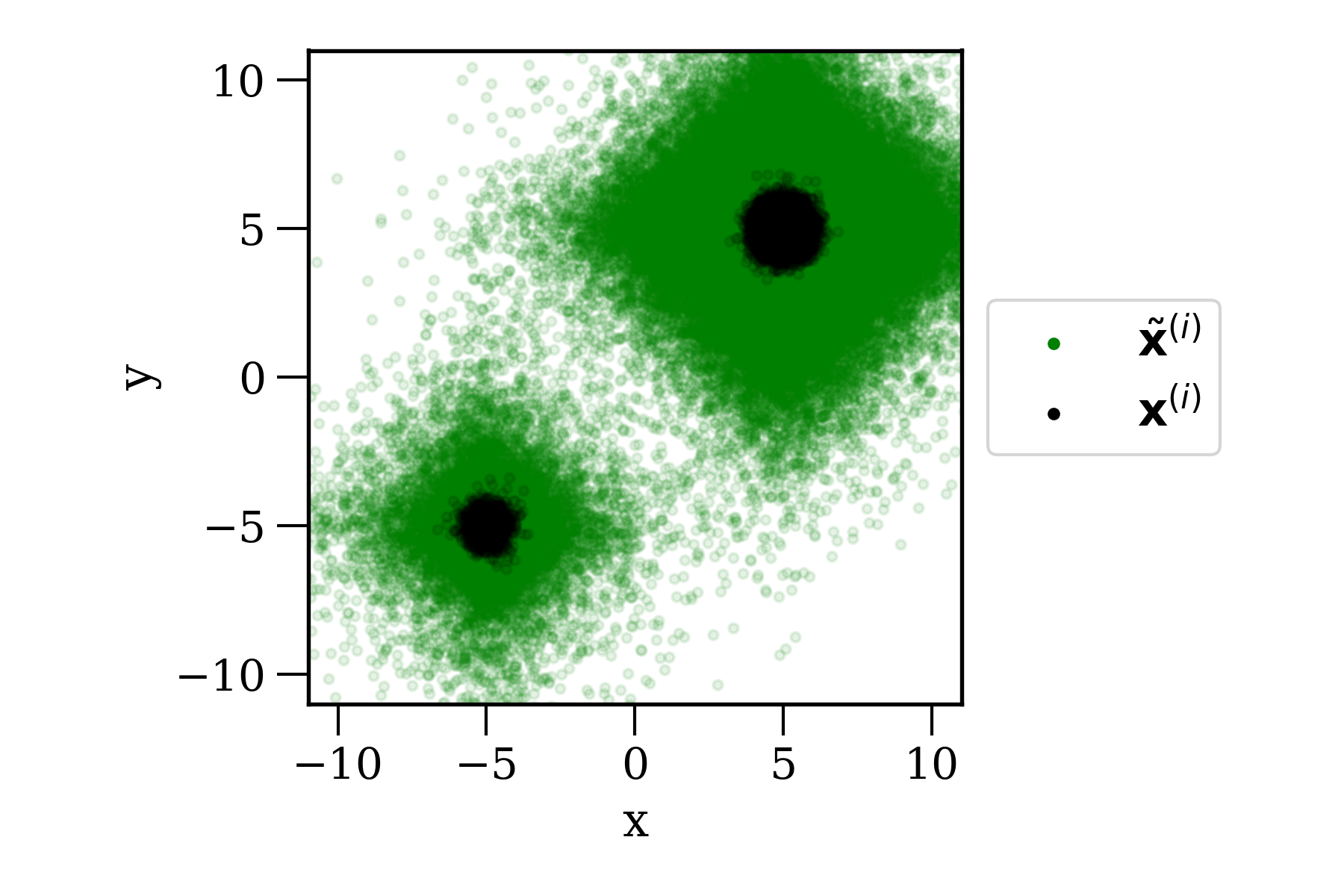}}%
    \subfigure[]{\label{fig:dsm_ald_laplace_10x_4}\includegraphics[width=0.5\linewidth]{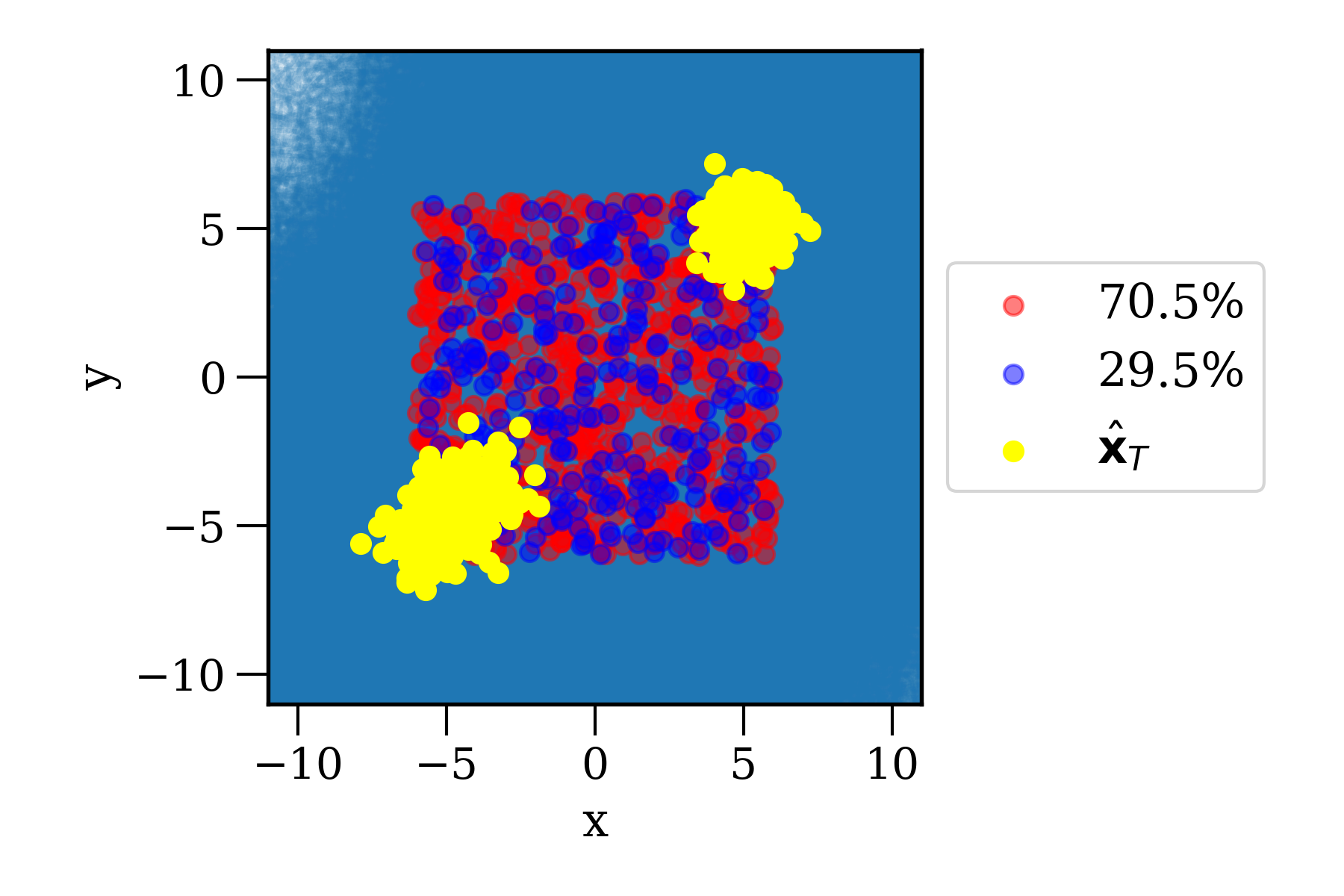}}}
    \caption{Laplace DSM with ALD. The setup and figures are identical to Figure~\ref{fig:dsm ald 10x example 2}, except that Laplace noise is used. \textbf{b} demonstrates that Laplace noise compensates for the class imbalance. $29.5\%$ of particles finishing in mode 2 even manages to overcompensate.}
    \label{fig:dsm ald laplace 10x 1}
\end{figure}

In addition, Figure~\ref{fig:dsm ald laplace 10x 1} illustrates Laplace DSM with ALD for the class imbalance problem of Figure~\ref{fig:dsm ald 10x example 2}, and \ref{fig:dsm_ald_laplace_10x_4} establishes that Laplace noise can compensate for class imbalance. In particular, $29.5\%$ of particles finishing in mode 2 means that HTDSM even manages to overcompensate.

To confirm that this trend is present for all $\beta<2$ noise types, Figure~\ref{fig:beta trend synthetic} extends this, repeating the experiment to estimate a confidence interval. Overall, the large jumps in sampling (similar to L\'{e}vy flights) free the sampling paths from being dominated by the more populous mode while preserving useful score estimates which point toward the underrepresented local maximum of the PDF.

Another insight is offered in Table~\ref{tab:sm vs diffusion}, where DSM and HTDSM are used with Gaussian or Laplace diffusion. As expected, models trained with standard DSM diverge when Laplace diffusion is used for sampling\footnote{Diverging here refers to approaching very large values which completely ignore the distribution modes (even if they are technically closer to one of the two).}. Also consistent with Figures~\ref{fig:beta trend synthetic} and \ref{fig:dsm ald 10x example 2}, DSM with Gaussian ALD suffers from mode collapse. Nevertheless, HTDSM can use sub-Gaussian diffusion to overcompensate for the asymmetric data, a trait that is valuable for realistic scenarios which often contain class imbalances. Finally, HTDSM can be used with Gaussian ALD solely as a method for providing better score estimates. This final process leads to the most accurate estimate of the imbalance in Table~\ref{tab:sm vs diffusion} and suggests that the way forward is to leverage the stability of Gaussian ALD alongside HTDSM gradients which are likely to be more accurate in low probability regions.
\begin{figure}[t]
    \centering
    \includegraphics[width=0.8\linewidth]{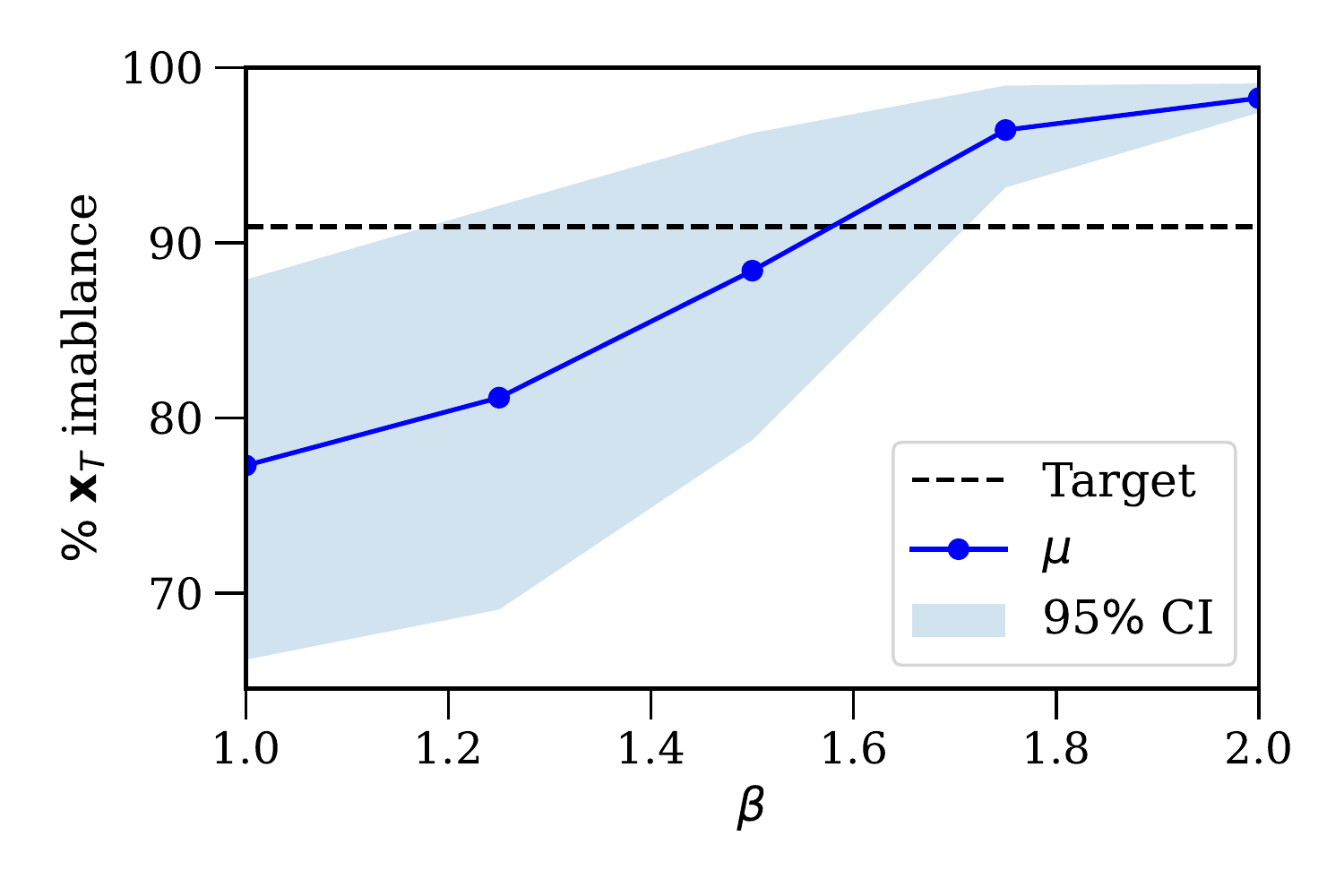}
    \caption{Mean percentage imbalance of generated data for $\beta\in[1,2]$ across 10 runs of the 10:1 experiment in Figure~\ref{fig:dsm ald 10x example 2}. Noise in training and sampling is sub-Gaussian. The 95\% confidence interval is bootstrapped from 10,000 resamples.}
    \label{fig:beta trend synthetic}
\end{figure}

\begin{table}[t]
    \centering
    \caption{Mean (plus lower/upper bounds for a 95\% confidence interval bootstrapped from 10,000 resamples) percentage imbalance of generated data for the 10:1 experiment in Figure~\ref{fig:dsm ald 10x example 2} with Gaussian or Laplace DSM or diffusion. DSM with Laplace noise at sampling time diverges regularly due to inaccurate gradients.}
    \label{tab:sm vs diffusion}
    \begin{tabular}{lcc}
        \toprule
         & Gaussian diffusion & Laplace diffusion\\
        \midrule
        DSM & 98.25 (97.40, 99.09) & Divergent \\
        HTDSM & \textbf{87.44 (82.85, 91.50)} & 77.28 (66.38, 87.79) \\
        \bottomrule
    \end{tabular}
\end{table}

The central take-aways of these synthetic experiments are:
\begin{itemize}
    \item HTDSM is an efficacious method of estimating a distribution's score function.
    \item Sub-Gaussian diffusion, causing L\'{e}vy-flight-like sampling paths, can overcome class imbalances, motivating extension to the continuous case (see Section~\ref{sec:Continuous extension to stochastic differential equations}). Sufficiently accurate and compensatory score estimates for these paths can also only be achieved with HTDSM.
    \item HTDSM with Gaussian diffusion offers a potentially even more general solution.
\end{itemize}

\subsection{High-dimensional class imbalances}\label{sec:High-dimensional class imbalances}

To extend analysis of how HTDSM mitigates class imbalances in higher dimensions, Figure~\ref{fig:imbalanced} present model generation results for a simplified version of the MNIST dataset. The data is limited to contain only the classes 1 and 8, which were chosen for their contrast in pixel space. The goal was to demonstrate how Gaussian DSM SBMs perform poorly with ALD in the presence of asymmetric class representation, by inducing an imbalance between classes 1 and 8. However, mode collapse occurred before the class ratio was even manipulated (further supported by the more expected imbalance in Figure~\ref{fig:mnist18_b-2.0_f-0.5_image_grid_1000} in Appendix \ref{sec:imbalanced_experiments}). Gaussian DSM suffering such issues in this minimal setting appears to contradict \citet{song2019generative}, where the motivation for combining DL and ALD was to overcome uneven mode weights. Moreover, it brings into question the cause of recent impressive generative results with SBMs, which may require the regularisation of many classes in the data to produce more general score estimates.

To reinforce this result, the same even-class model was re-ran to produce 100 samples\footnote{Models were also retrained to verify that this issue was reproducible.}. DSM with 100 steps per level (s/l) produced six ones with $P(6)<10^{-21}$ under a binomial model, whereas HTDSM produced eighteen ones with $P(18)<10^{-10}$, a massive relative improvement. Since the generated HTDSM images were speckled (see Figure\ref{fig:sbm mnist 1/8 b=1} in Appendix \ref{sec:imbalanced_experiments}, the same experiment was repeated with 1,000s/l. This revealed that more sampling steps alleviates Gaussian DSM imbalance almost completely, producing 48 ones, and for HTDSM 32 ones were generated (well within two standard deviations of the normal approximation to the underlying binomial distribution here). Therefore, one conclusion is that HTDSM is beneficial for varied sampling in compute-constrained scenarios and that avoidance of mode collapse in the literature may be, in part, due to intensive sampling procedures at high noise levels.

\begin{figure}[t]
    \centering
     \subfigure[]{\label{fig:mnist18_b-2.0_image_grid_1000}\includegraphics[width=0.35\linewidth]{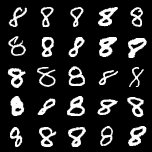}} \hfill%
    \subfigure[]{\label{fig:mnist18_b-1.0-10x}\includegraphics[width=0.35\linewidth]{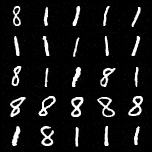}}
    \caption{(\textbf{a}) Generated samples from a Gaussian DSM model trained for 20,000 steps on digits 1 and 8 from the MNIST dataset and sampled with 100 steps per level (s/l) of Gaussian ALD. No digits resembling a 1 are present for 25 samples, indicating a sampling process which induces the class imbalance. (\textbf{b}) Generated samples from a HTDSM model trained in-line with (a), and sampled with 1000 s/l, indicate more even class balance.}
    \label{fig:imbalanced}
\end{figure}



\subsection{High-dimensional unconditional generation}\label{sec:High-dimensional unconditional generation}

Tables~\ref{tab:sbm mnists_main} and \ref{tab:sbm cifar_celeba_main} summarise the DGM metrics attained using HTDSM on the MNIST, Fashion-MNIST, CIFAR-10, and CelebA datasets. For MNIST and Fasion-MNIST we report precision, recall, density, and coverage estimated on 5000 generated samples using the Python \texttt{prdc}\footnote{\href{https://github.com/clovaai/generative-evaluation-prdc}{https://github.com/clovaai/generative-evaluation-prdc}} package with $k=5$ \cite{naeem2020reliable}. For CIFAR-10 and CelebA we report metrics, reliant upon Inception v3 (Inception Score and Kernel Inception Distance), computed on 5000 (for CIFAR-10) and 1000 (for CelebA) generated samples and compared to the respective training dataset using the Python \texttt{torch-fidelity}\footnote{\href{https://github.com/toshas/torch-fidelity}{https://github.com/toshas/torch-fidelity}} package. IS and KID for MNIST and Fashion-MNIST are given in Tables \ref{tab:sbm mnist all} and \ref{tab:sbm fashion all} in Appendix~\ref{sec:tab_results}.


For HTDSM on MNIST with $\beta=1.5$, precision, recall, and coverage were found to improve over standard DSM. However, all other metrics did not improve and Inception-based metrics are markedly worse. Although this suggests lower perceptual quality, Figure~\ref{fig:sbm mnist unconditional} in Appendix~\ref{sec:samples} seems to refute this. In particular, Figure~\ref{fig:mnist_b-2.0_image_grid_200000} depicts Gaussian ALD failing to generate a single digit similar to a one, and the probability of the seventeen zeros occurring in the real dataset is less than $10^{-7}$, so problems persist. In Tables~\ref{tab:sbm mnists_main} and \ref{tab:sbm cifar_celeba_main}, $\beta\neq2.0$ dominates the majority of the metrics, demonstrating that non-Gaussian DSM is advantageous for certain datasets. Overall, $\beta<2$ is a promising direction for future research and larger-scale experiments, but $\beta=1.0$ suffered from convergence issues at the scale of CelebA. It is also possible to explore the effects of light-tailed DSM by setting $\beta=2.5$, as this corresponds to estimating a score function which is very large for high noise (see Figure~\ref{fig:gn score} for intuition). However, diffusion convergence was often found to be too quick, resulting in cartoon-like final images with strong features and no subtleties.

\begin{table*}[t]
    \centering
    \caption{DGM metrics for unconditional samples from a model trained with HTDSM, and sampled from using Gaussian ALD, for different values of $\beta$ on the MNIST and Fashion-MNIST datasets. Arrows indicate that higher ($\uparrow$) metric values are better.}
    \begin{adjustbox}{width=.8\textwidth}
    \begin{tabular}{lrrrr|rrrr}
        \toprule
        & \multicolumn{4}{c}{MNIST} & \multicolumn{4}{c}{Fashion-MNIST} \\
            \cmidrule(lr){2-5}
            \cmidrule(lr){6-9}
         & $\beta=1.0$ & $\beta=1.5$ & $\beta=2.0$ & $\beta=2.5$ & $\beta=1.0$ & $\beta=1.5$ & $\beta=2.0$ & $\beta=2.5$ \\
        \midrule
        Precision $\uparrow$ & 0.9417 & \textbf{0.9244} & 0.912 & 0.894 & 0.1244 & 0.922 & 0.884 & \textbf{0.9230} \\
        Recall $\uparrow$ & 0.8634 & 0.9023 & \textbf{0.936} & 0.905 &\textbf{ 0.9592} & 0.765 & 0.787 & 0.7754 \\
        Density $\uparrow$ & \textbf{0.9869} & 0.9210 & 0.867 & 0.849 & 0.0355 & 1.541 & 1.406 & \textbf{1.600} \\
        Coverage $\uparrow$ & \textbf{0.9112} & 0.7816 & 0.780 & 0.733 & 0.0351 & \textbf{0.651} & 0.587 & 0.5962 \\
        \bottomrule
    \end{tabular}
    \end{adjustbox}
    \label{tab:sbm mnists_main}
\end{table*}



\begin{table*}[t]
    \centering
     \caption{DGM metrics for unconditional samples from a model trained with HTDSM, and sampled from using Gaussian ALD, for different values of $\beta$ on the Cifar-10 and CelebA datasets.}
    \begin{adjustbox}{width=.8\textwidth}
    \begin{tabular}{lrrr|rrr}
        \toprule
         & \multicolumn{3}{c}{Cifar-10} & \multicolumn{3}{c}{CelebA} \\
            \cmidrule(lr){2-4}
            \cmidrule(lr){5-7}
         & $\beta=1.5$ & $\beta=2.0$ & $\beta=2.5$ &  $\beta=1.5$ & $\beta=2.0$ & $\beta=2.5$ \\
        \midrule
        IS $\uparrow$ &  $\V{8.209\pm0.102}$ & $8.09\pm0.029$ & $7.026\pm0.074$ 
                      & $1.977\pm0.130$ & $2.070\pm0.125$ & $\V{2.111\pm0.09}$ \\
        KID $\downarrow$ & $0.009\pm0.001$ & $\V{0.007\pm0.001}$ & $0.016\pm0.001$
                         & $\V{0.095\pm0.001}$ & $0.117\pm0.001$ & $0.129\pm0.001$ \\
        \bottomrule
    \end{tabular}
    \end{adjustbox}
    \label{tab:sbm cifar_celeba_main}
\end{table*}

\begin{figure}[h]
    \centering
     \subfigure{\includegraphics[width=0.32\linewidth]{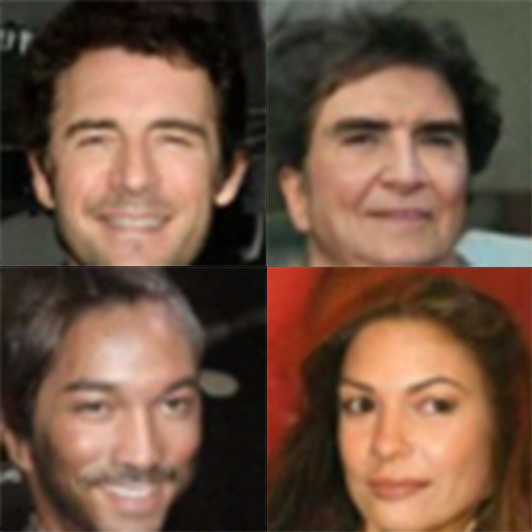}} \hfill%
    \subfigure{\includegraphics[width=0.32\linewidth]{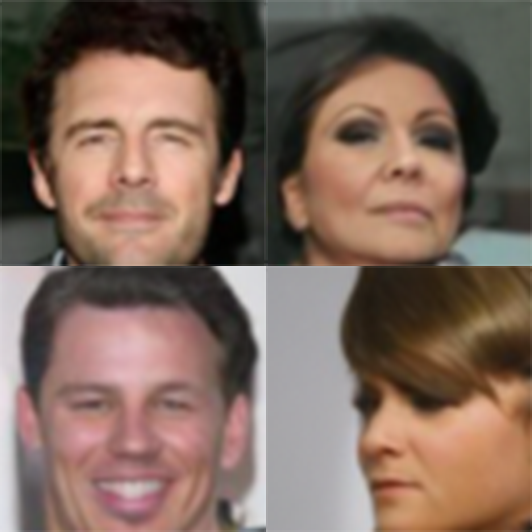}} \hfill%
    \subfigure{\includegraphics[width=0.32\linewidth]{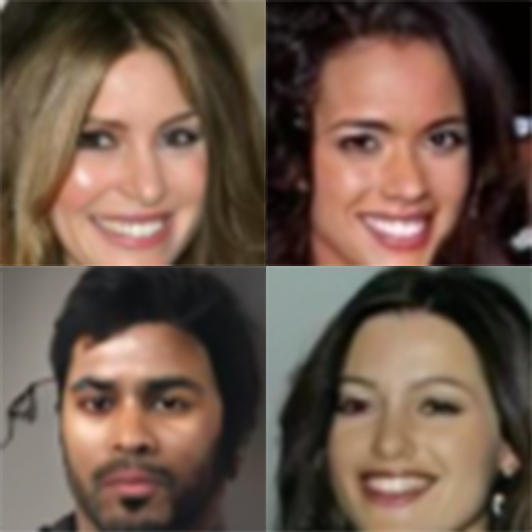}} \hfill%
    \caption{CelebA samples from HTDSM model trained with\\ $\beta=1.5,2,2.5$, respectively.}
    \label{fig:celeba_samples}

\end{figure}

\section{Conclusion}\label{sec:Chapter 4 - Summary}

We provided a thorough expansion of the theory behind DSM for SBMs. This foothold was used as the basis for novel theoretical expansion of DSM to heavy-tailed DSM, noising and denoising with the family of generalised normal distributions. Insight into the undesirable $n$-dimensional annuli in Gaussian DSM, as well as an understanding of the generalised normal score function for $\beta<2$, motivated the use of heavier tails. GN noise was then found to concentrate in a skewed distribution which prompted a general algorithm to choose a noise scaling sequence in Section~\ref{sec:Choosing a scale parameter sequence for arbitrary noise}.

Our examples demonstrated the propensity of ALD with Gaussian DSM to suffer from mode collapse. The $2D$ example outlined differences at both training and sampling time for DSM and HTDSM. The latter was shown to be a tenable alternative method of estimating a distribution's score function. Experiments suggested that heavy-tailed noise can always be scaled down to dampen sampling with jumps, whereas limitations such as class imbalances are inherent to the data. HTDSM with Gaussian diffusion offered the most general method of learning and sampling, balancing better gradients in low probability regions with well-behaved diffusion. HTDSM also continued to provide promising results when scaling to higher dimension datasets. $1<\beta<2$ appears to be a relatively stable type of GN noise which offers improved image generation across a range of metrics in Section~\ref{sec:High-dimensional unconditional generation}. Moreover, sub-Gaussian diffusion, causing L\'{e}vy-flight-like sampling paths, can overcome class imbalances, motivating extension to the continuous case (see Section~\ref{sec:Continuous extension to stochastic differential equations}), while sufficiently accurate and compensatory scores for these paths can only be estimated with HTDSM.

\clearpage
\bibliography{main_arxiv}
\bibliographystyle{icml2021}

\clearpage

\appendix
\onecolumn

\section{Experimental setup}\label{sec:Experimental setup}

\subsection{Training}\label{sec:Training}

For benchmark image datasets, the experimental setup is described for reproducibility. Different neural network architectures were used for different datasets, with respective sizes positively correlated. For MNIST, Fashion-MNIST, and CIFAR-10 ($28\times28\times1$ or $32\times32\times3$), the ResNet architecture \cite{he2016deep} from \citet{song2020improved} is used in-line with the literature. In the case of CIFAR-10, one convolutional layer is added during spatial down and up-sampling to reflect the more complex data. Then, for any larger datasets, experiments also follow in-line with the literature, but are limited by lower levels of compute. NN architecture specifics and further training information are summarised in Appendix B of \citet{song2020improved}. We note that we were relatively compute constrained throughout experimentation and therefore assess relative, rather than absolute, performance.

When training each SBM, multiple noise levels were used to allow for ALD at sampling time. Each noise level adds generalised normal distribution noise to the image, scaled by the constant factored into the derivations in Section~\ref{sec:High dimensional noising}. For an input image and its sampled noise, the generalised normal score is calculated according to \eqref{eq:gn_score}, and set as the model target. The time complexity impact of sampling $\mathcal{GN}(\mu=x,\alpha=1,\beta)$ noise at scale is minimal, as the sampling procedure for each dimension is simply
\begin{align}
    \gamma&\sim\textrm{Gamma}\left(\textrm{shape}=1+1/\beta, \textrm{rate}=2^{-\beta/2}\right) \\
    \delta &= \alpha\gamma^{1/\beta}/\sqrt{2}\\
    \hat{x}&\sim\mathcal{U}(\mu-\delta,\mu+\delta),
\end{align}
following \citet{choy2003extended}, where $\mathcal{U}$ denotes the uniform distribution.

At sampling time, in-line with \citet{song2019generative}, images are initialised by a uniform distribution over the pixels, before ALD iteration begins. ALD (see Algorithm~\ref{alg:Anneal langevin dynamics}), uses the multiple noise levels ${\sigma_{1},\ldots,\sigma_{k}}$ from training time, with a learning rate proportional to the ratio of the squared current noised level to the squared maximum noise level. Each noise level iterates for step limits ranging from 10 to 500, depending on both the dataset and the value of $\beta$---the latter due to the low absolute score values for distant noise when $\beta<1.5$ (see Figure~\ref{fig:gn score}).

\subsection{Langevin dynamics}\label{sec:Langevin dynamics}
Samples are generated using score estimates to ascend the gradient for a given input. Due to the monotonic nature of the logarithm function, iteratively following the direction of the largest score estimate is equivalent to performing gradient ascent on the data distribution. As such, multiple iterative optimisation algorithms are available for use at this stage.

The procedure of choice, Langevin Monte Carlo (LMC, \cite{besag1994comments}), is a Markov Chain Monte Carlo (MCMC) method for obtaining random samples from probability distributions for which direct sampling is difficult. The goal is to follow the gradient but add a bit of noise so as to not get stuck at local optima, explore the entire distribution, and sample from it.


\begin{algorithm}
\caption{Langevin dynamics.}\label{alg:Langevin dynamics}
\begin{algorithmic}
\STATE{\textbf{Input}} Hyperparameters of prior $\pi(\V{x})$, step size $\varepsilon\ll1$, step limit $T$, and initialised $\Tilde{\V{x}}_{0}\sim\pi(\V{x})$.
    \FOR{$t=1\ldots T$}
        \STATE {\small Sample:\\ $\V{z}_{t}\sim\mathcal{N}(0,I)$
        $\Tilde{\V{x}}_{t} = \Tilde{\V{x}}_{t-1} + \varepsilon\nabla_{\V{x}}\log p\left(\Tilde{\V{x}}_{t-1}\right) + \sqrt{2\varepsilon}\V{z}_{t}$}
    \ENDFOR
\end{algorithmic}
\end{algorithm}

\begin{algorithm}
\caption{Annealed Langevin dynamics.}\label{alg:Anneal langevin dynamics}
\begin{algorithmic}
\STATE{\textbf{Input}} Gaussian noise scaling factors $\{\sigma_{1},\ldots,\sigma_{k}\in\mathbb{R}_{+} \ s.t. \ \sigma_{1}>\ldots>\sigma_{k}\}$, and parameters to run LD.
    \FOR{$i=2\ldots k$}
        \STATE {\small Run LD$_{i}$ with noise level $\sigma_{i}$ starting from the result of LD$_{i-1}$}
    \ENDFOR
\end{algorithmic}
\end{algorithm}

\subsection{Deep generative model metrics.}\label{sec:DGM metrics}
Justification of metrics providing a reasonable and consistent evaluation of images synthesised by generative models is far from a solved problem. Current popular metrics often make use of the final or penultimate layer activations of a heavily-trained convolutional network, such as the Inception v3 model of \citet{szegedy2015going}. Despite obvious bias toward generative models trained on similar datasets and in similar manners, as well as the plethora of more performant models since Inception v3 was trained in 2015, these metrics persist, should be used for comparison with the literature, and are now described. Throughout this paper, 5,000 samples are used to calculate each DGM metric.

When comparing generated samples to one another, the \textit{Inception score} \cite{salimans2016improved}
\begin{align}
    \textrm{IS} = \exp(\mathbb{E}_{\V{x}\sim\mathcal{G}}[\textrm{KL}(p(\V{y}|\V{x})\parallel p(\V{y}))]),
\end{align}
intuitively rewards low entropy classification of generated samples ($\V{x}\sim\mathcal{G}$), as well as variation. Alternatively, the popular \textit{Fr\'{e}chet inception distance} (FID, \cite{heusel2017gans}) compares Inception v3 activation statistics between generated samples and the samples used to train the generative model, requiring thousands of new samples at evaluation time.
\begin{align}
    \textrm{FID} = \lVert\mu_{\mathcal{R}}-\mu_{\mathcal{G}}\rVert_{2}^{2} + \textrm{tr}\left(\Sigma_{\mathcal{R}} + \Sigma_{\mathcal{G}} - 2(\Sigma_{\mathcal{R}}\Sigma_{\mathcal{G}})^{1/2}\right),
\end{align}
where these values are activation statistics and $\mathcal{R}$ refers to the real dataset. The FID approach was taken further by \textit{kernel inception distance} (KID, \cite{binkowski2018demystifying})
\begin{align}
    k(\V{x}_{1},\V{x}_{2}) &= \left(\V{x}_{1}^{T}\V{x}_{2}/n+1\right)^{3}\\
    K(\V{x}_{1},\V{x}_{2}) &= k(\phi_{\textrm{\ I-v3}}(\V{x}_{1}),\phi_{\textrm{\ I-v3}}(\V{x}_{2})),
\end{align}
where $\phi_{\textrm{\ I-v3}}(\cdot)$ maps to an $nD$ Inception v3 layer, the cubic exponent accounts for skew, crucially no parametric form for the distribution is assumed, and an average over all real-fake pairs is taken.

An alternative approach to sample quality assessment is to directly calculate distribution overlap. In \citet{sajjadi2018assessing}, the authors used local $n$-balls to form high-dimensional equivalents of \textit{precision} and \textit{recall}, avoiding pathological examples of models with equal FID but visually juxtaposed sample quality. This idea was later extended to the more localised and precise \textit{density} and \textit{coverage} metrics of \citet{naeem2020reliable}, where neighbourhoods are instead built from the $k$ nearest neighbours. The mathematical definitions of these concepts can be found in \citet{naeem2020reliable}.

\section{Derivations}

\subsection{Unintuitive high-dimensional statistics}\label{sec:Unintuitive high-dimensional statistics}

To demonstrate unintuitive statistical behaviour in high-dimensional space, consider two classic examples:
\begin{itemize}
    \item For an $iid$ random vector $\mathbf{X}=[X_{1},\ldots,X_{n}]^{\textrm{T}}$ with $X_{i}\sim U(a,b)$, the majority of probability mass resides in the corners of the hypercube for high $n$.
    \item For an $iid$ random vector $\mathbf{X}=[X_{1},\ldots,X_{n}]^{\textrm{T}}$ with $X_{i}\sim \mathcal{N}(\mu,\sigma)$, the probability mass concentrates in a thin annulus (shell).
\end{itemize}
Both of these phenomena are instances of the \textit{concentration of measure}---the principle that a random variable that depends in a Lipschitz way on many independent variables is essentially constant \cite{talagrand1996new}. As the latter spherical case motivates this work, it is now explored in full.

Consider the \textit{squared} $L^{2}$ \textit{norm distribution} of the isotropic Gaussian vector
\begin{align}
    Y = \lVert\mathbf{X}\rVert_{2}^{2},
\end{align}
which, by independence, gives
\begin{align}
    Y = \sum\limits_{i=1}^{n}X_{i}^{2} = nX_{i}^{2}.
\end{align}
Changing random variables
\begin{align}
    F_{Y}(y) &= P(Y\leq y)\\
    &= P(nX_{i}^{2}\leq y)\\
    &= P\left(|X_{i}|\leq\sqrt{\frac{y}{n}}\right)\\
    &= \Phi\left(\sqrt{\frac{y}{n}}\right) - \Phi\left(-\sqrt{\frac{y}{n}}\right),
\end{align}
where $\Phi$ is the Gaussian CDF, and differentiating gives the PDF of $Y$
\begin{align}
    f_{Y}(y) = F_{Y}'(y) &= \frac{1}{2\sqrt{ny}}\phi\left(\sqrt{\frac{y}{n}}\right) - \frac{1}{2\sqrt{ny}}\phi\left(-\sqrt{\frac{y}{n}}\right)\\
    &= \frac{1}{\sqrt{ny}}\phi\left(\sqrt{\frac{y}{n}}\right)\\
    &= \frac{1}{\sqrt{ny}}\frac{1}{\sqrt{2\pi}}\exp\left\{-\frac{y}{2n}\right\},
\end{align}
where $\phi$ is the Gaussian PDF. Rewriting this expression and using $\sqrt{\pi}=\Gamma(1/2)$
\begin{align}
    f_{Y}(y) &= \frac{1}{\sqrt{\pi}\sqrt{2n}}y^{-\frac{1}{2}}\exp\left\{-\frac{y}{2n}\right\}\\
    &= \frac{1}{\Gamma(1/2)(2n)^{\frac{1}{2}}}y^{\frac{1}{2}-1}\exp\left\{-\frac{y}{2n}\right\},
\end{align}
recovers the gamma distribution PDF
\begin{align}
    W\sim\textrm{Gamma}(k,\theta)\implies f(w;k,\theta)=\frac{1}{\Gamma(k)\theta^{k}}w^{k-1}e^{-\frac{k}{\theta}},
\end{align}
and demonstrates that $Y$ is distributed as either of
\begin{align}
    Y\sim\textrm{Gamma}\left(k=\frac{1}{2},\theta=2n\right) = n\textrm{Gamma}\left(\frac{1}{2},2\right),
\end{align}
and therefore follows either of the chi-squared distributions
\begin{align}
    Y\sim \chi^{2}(n) = n\chi^{2}(1),
\end{align}
which approaches a Gaussian distribution centred at $n$ in the limit $n\to\infty$. See Section~\ref{sec:Comparison of concentration moments} for a full description of the moments and Figure~\ref{fig:chi-squared} for a visualisation of the chi-squared distribution for increasing degrees of freedom.

\subsection{Scale parameter sequences for arbitrary noise distributions}\label{sec:Choosing a scale parameter sequence for arbitrary noise}

It is now apparent that the generalised normal distribution can be used for denoising score matching, leading to thicker concentric annuli. However, the non-zero skew in \eqref{eq:chi2 skew}, representing the asymmetric norm distribution for low $n$, cannot necessarily be ignored. This asymmetry implies that the spacing of noise levels using variance in \cite{song2020improved}, is inaccurate.

In particular, the probability mass in the left/right tail of one noise level annulus will be larger than the probability mass in the right/left tail, respectively, of the adjacent annulus. In practical terms, this means overly-dense concentric noise levels in ALD. Although the time needed to sample each noise-level per training iteration will not be affected, this will increase overall training time, sampling (generation) time, and render some sampling steps redundant. As the same problem extends to the generalised noise characterised in this paper, it is vital to examine the skew of the norm distribution from \eqref{eq:generalised nomal length distribution}.

To begin, note the $r$th raw moment of $Y\sim\mathcal{GG}(a,d,p)$ is
\begin{align}
    \mathbb{E}\left[Y^{r}\right]=a^{r}\frac{\Gamma((d+r)/p)}{\Gamma(d/p)},
\end{align}
implying that, for $\lVert\mathbf{X}\rVert_{2}^{2}=Y\sim~\mathcal{GG}(a=n,d=1/2,p=\beta/2)$
\begin{align}
    \mathbb{E}\left[Y^{3}\right] = n^{3}\frac{\Gamma(7/\beta)}{\Gamma(1/\beta)}.
\end{align}
Then, using the $3^{rd}$ central moment expansion for skew
\begin{align}
    \textrm{Skew}(Y) &= \mathbb{E}\left[\left(\frac{Y-\mu}{\sigma}\right)^{3}\right]\\
    &= \frac{1}{\sigma^{3}}\left(\mathbb{E}\left[Y^{3}\right] - 3\mu\mathbb{E}\left[Y^{2}\right] + 3\mu^{2}\mathbb{E}[Y] - \mu^{3}\right)\\
    &= \frac{1}{\sigma^{3}}\left(\mathbb{E}\left[Y^{3}\right] - 3\mu\sigma^{2} - \mu^{3}\right)\\
    &= \frac{1}{n^{3}C_{2}^{\frac{3}{2}}}\left\{n^{3}\frac{\Gamma(7/\beta)}{\Gamma(1/\beta)} -3nC_{1}n^{2}C_{2} - n^{3}C_{1}^{3}\right\}\\
    &= \frac{1}{C_{2}^{\frac{3}{2}}}\left\{\frac{\Gamma(7/\beta)}{\Gamma(1/\beta)} -3C_{1}C_{2} - C_{1}^{3}\right\},
\end{align}
where $C_{1}$ and $C_{2}$ are defined as in \eqref{eq:mean_scaling} and \eqref{eq:var_scaling} respectively. It is interesting to note that this skew expression is constant with respect to dimension.

The tangible Laplace noising, $\beta=1$, case can now be exemplified, as
\begin{align}
    C_{1} &= \frac{\Gamma(3)}{\Gamma(1)} = 2 \\
    C_{2} &= \frac{\Gamma(5)}{\Gamma(1)} - \left(\frac{\Gamma(3)}{\Gamma(1)}\right)^{2} = 20,
\end{align}
and it is clear that
\begin{align}
    \textrm{Skew}(Y;\beta=1) &= 20^{-3/2}\left(720 - 3\times2\times20 - 2^{3}\right)\\
    &= 74/\sqrt{5}^{3}\\
    &\approx 6.19,
\end{align}
so the resulting distribution is very positively skewed.

Despite the disappointing prospects of this result, the potential for large asymmetric annulus overlap, it also motivates a better understanding of the general case. How can concentric annuli be constructed with equal overlapping probability mass?

To motivate the general algorithm, assess the case where the generalised normal noise is scaled by an arbitrary noise level $\sigma_{i}$.
\begin{align}
    X_{i}/\sigma_{i}\sim\mathcal{GN}(\mu=0,\alpha=1,\beta)
    \implies X_{i}\sim\mathcal{GN}(0, \sigma_{i},\beta),
\end{align}
it is, therefore, true that for GN noise vector $\V{X}$
\begin{align}
    \lVert \V{X}\rVert_{2}^{2}=Y\sim\mathcal{GG}\left(n\sigma_{i}^{2},1/2,\beta/2\right). \label{eq:length dist for noise level}
\end{align}
In an ascending sequence of noise where the goal is to calculate $\sigma_{i+1}$ from $\sigma_{i}$ with a given probability mass overlap, the quantile function of the norm distribution must then be used by inverting the corresponding CDF. Here, the CDF is
\begin{align}
    F_{\mathcal{GG}}(x;a,d,p) = \frac{\gamma\left(d/p,(x/a)^{p}\right)}{\Gamma(d/p)}, \label{eq:gg cdf}
\end{align}
where $\gamma(\cdot)$ is the \textit{lower incomplete gamma function}
\begin{align}
    \gamma(s,x) = \int_{0}^{x} t^{s-1}e^{-t}dt.
\end{align}
Although \eqref{eq:gg cdf} appears difficult to invert, due to the \textit{inverse of composite functions}, the quantile function for quantile $q$ follows as
\begin{align}
    F_{\mathcal{GG}}^{-1}(q;a,d,p) = a\left[G^{-1}(q)\right]^{1/p}, \label{eq:gg quantile function}
\end{align}
where
\begin{align}
    G(x)=F_{\textrm{Gamma}}\left(x;\alpha'=d/p,\beta'=1\right) &= \frac{\gamma(\alpha',\beta' x)}{\Gamma(\alpha')}\\
    &= \frac{\gamma(d/p,x)}{\Gamma(d/p)}, \label{eq:gamma cdf}
\end{align}
a scaled Gamma distribution CDF (Greek letters here are for the \textit{standard} Gamma distribution), and the form $\gamma(c_{1},c_{2}x)/\Gamma(c_{1}),\ c_{1},c_{2}\in\mathbb{R}_{+},$ is known as the \textit{regularised gamma function}.

Finally, substituting \eqref{eq:gamma cdf} into \eqref{eq:gg quantile function}, gives
\begin{align}
    F_{\mathcal{GG}}^{-1}(q;a,d,p) = a\left(\left[\frac{\gamma(d/p,q)}{\Gamma(d/p)}\right]^{-1}\right)^{1/p},
\end{align}
before substituting \eqref{eq:length dist for noise level} as well provides
\begin{align}
    F^{-1}_{Y}(q) = n\sigma_{i}^{2}\left(\left[\frac{\gamma(1/\beta,q)}{\Gamma(1/\beta)}\right]^{-1}\right)^{2/\beta}.
\end{align}

After these steps, for an example overlap of 5\% of probability mass, it is now possible to say
\begin{align}
    q_{i,0.95} = n\sigma_{i}^{2}\left(\left[\frac{\gamma(1/\beta,0.95)}{\Gamma(1/\beta)}\right]^{-1}\right)^{2/\beta},
\end{align}
is the upper quantile for $\sigma_{i}$. Crucially, this expression can be inverted to obtain $\sigma_{i+1}$ by setting the next lower bound equal to the current upper bound, $q_{i+1,0.05}=q_{i,0.95}$ and inverting
\begin{align}
    \sigma_{i+1} = \sqrt{\frac{q_{i,0.95}}{n}}\left(\left[\frac{\gamma(1/\beta,0.05)}{\Gamma(1/\beta)}\right]^{-1}\right)^{-1/\beta}.
\end{align}
These last two equations outline a general procedure for consecutive noise levels with equal distribution overlap, detailed fully in Algorithm~\ref{alg:noise scales}. Of practical significance is the Python function \texttt{scipy.special.gammaincinv} which numerically estimates the troublesome inverse regularised gamma function to arbitrary precision \cite{gil2012efficient}.

\begin{algorithm}
\caption{Scale parameter sequence generation}\label{alg:noise scales}
\begin{algorithmic}
\STATE{\textbf{Input}} Fixed hyperparameters of piecewise log-differentiable noise distribution, non-overlapping distribution proportion $\delta \in (0, 1)$, small initial noise level $\sigma_{1} > 0$, and large final noise level $\sigma_{\max}$.
\STATE{\textbf{Initialise}} $i=1$, $q_{i}^{l}=0$, and $q_{i}^{u}=0$.
    \WHILE{$q_{i}^{u}<\sigma_{\max}$}
        \STATE {Calculate upper quantile $q_{i}^{u}=Q_{norm}\left(\sigma_{i},\frac{1+\delta}{2}\right)$ of $nD$ norm distribution}
        \STATE {Calculate scaling needed to equate to lower quantile $q_{i+1}^{l}=q_{i}^{u}$}
        \STATE {$\sigma_{i+1} = Q_{norm}^{-1}\left(\sigma_{i},\frac{1-\delta}{2}\right)$}
        \STATE {$i = i+1$}
    \ENDWHILE
\end{algorithmic}
\end{algorithm}

\subsection{Continuous extension to stochastic differential equations}\label{sec:Continuous extension to stochastic differential equations}

Given the success of multiple noise scales in Gaussian ALD, recent SBM continuations have considered infinitely many noise levels, such that the perturbed data distributions evolve according to a stochastic differential equation (SDE). The goal is to construct a diffusion process $\{\V{x}(t)\}_{t=0}^{T}$, $t\in[0,T]$, such that $\V{x}(0)\sim p_{0}$ is the dataset of i.i.d. samples, and $\V{x}(T)\sim p_{T}$ is the prior distribution, with a tractable form to generate samples efficiently. This diffusion process can be modelled as the solution to an It\^{o} SDE\footnote{A full description of the methods for calculus on stochastic processes, the foremost being It\^{o} and Stratonovich calculus, can be found in \cite{sarkka2019applied}.}
\begin{align}
    d\V{x} = \V{f}(\V{x},t)dt + g(t)d\V{w}, \label{eq:diffusion process}
\end{align}
where $\V{w}$ is the standard Wiener process (Brownian motion), $\V{f}(\cdot,t):\mathbb{R}^{n}\to\mathbb{R}^{n}$ is a vector-valued function called the \textit{drift} coefficient of $\V{x}(t)$, and $g(\cdot):\mathbb{R}\to\mathbb{R}$ is a scalar function known as the \textit{diffusion} coefficient of $\V{x}(t)$. Here, the diffusion coefficient is assumed to be a scalar (instead of a $d\times d$ matrix) and does not depend on $\V{x}$. The SDE has a unique strong solution as long as the coefficients are globally lipschitz in both state and time \cite{oksendal2003stochastic}. Henceforth, the
probability density of $\V{x}(t)$ is denoted by $p_{t}(\V{x})$, and $p_{st}(\V{x}(t)|\V{x}(s))$ denotes the transition kernel from $\V{x}(s)$ to $\V{x}(t)$, where $0\leq s<t\leq T$. Typically, $p_{T}$ is an unstructured prior distribution that contains no information about $p_{0}$.

It is possible to start from samples of $\V{x}(T)\sim p_{T}$ and reverse the process to obtain samples from $\V{x}(0)\sim p_{0}$. The main result in \citet{anderson1982reverse} states that the reverse of a diffusion process is a also a diffusion process running backwards in time and given by the reverse-time SDE
\begin{align}
    d\V{x} = \left[\V{f}(\V{x},t) - g(t)^{2}\nabla_{\V{x}}\log p_{t}(\V{x})\right]dt + g(t)d\V{\Bar{w}}, \label{eq:reverse diffusion process}
\end{align}
where $\V{\Bar{w}}$ is a reverse-time Wiener process and $\nabla_{\V{x}}\log p_{t}(\V{x})$ is estimated by
\begin{align}
    \theta^{*} = \argmin_{\theta} \mathbb{E}_{t}\left\{\lambda(t)\mathbb{E}_{\V{x}(0)}\mathbb{E}_{\V{x}(t)|\V{x}(0)}\left[\lVert s_{\theta}(\V{x}(t),t) - \nabla_{\V{x}(t)}\log p_{0t}(\V{x}(t)|\V{x}(0))\rVert_{2}^{2}\right]\right\},
\end{align}
for $\lambda:[0,T]\to\mathbb{R}_{+}$ a positive weighting function, $t\sim\mathcal{U}(0,T)$, $\V{x}(0)\sim p_{0}(\V{x})$, and $\V{x}(t)\sim p_{0t}(\V{x}(t)|\V{x}(0))$. The overall process was given the general name score matching Langevin dynamics (SMLD) in \citet{song2020score}.

When using $N$ noise scales, each perturbation kernel $p_{\sigma_{i}}(\V{x}|\V{x}_{0})$ of SMLD can be derived from the Markov chain
\begin{align}\label{eq:discrete SMLD}
    \V{x}_{i} = \V{x}_{i-1} + \sqrt{\sigma_{i}^{2} - \sigma_{i-1}^{2}}\V{z}_{i-1},
\end{align}
where $i=1,\ldots,N$ and $\V{z}_{i-1}\sim\mathcal{N}(\V{0},\V{I})$, $\V{x}_{0}\sim p_{\textrm{data}}$, and $\sigma_{0}=0$ is used to simplify notation. Whereas \citet{song2020score} proceed with Gaussian noise, the continuation with sub-Gaussian noise is now assessed.

To begin, let the elements of $\V{l}_{i-1}$ follow a sub-Gaussian distribution. Then let $\V{x}(i/N)=\V{x}_{i}$, $\sigma(i/N)=\sigma_{i}$, and $\V{l}(i/N)=\V{l}_{i}\ \forall i$. With $\Delta t=1/N$, it is then possible to write
\begin{align}
    \V{x}(t+\Delta t) &= \V{x}(t) + \sqrt{\sigma^{2}(t+\Delta t)-\sigma^{2}(t)}\ \V{l}(t)\\
    &\approx \V{x}(t) + \sqrt{\frac{d\left[\sigma^{2}(t)\right]}{dt}\Delta t}\ \V{l}(t),
\end{align}
where the approximate equality holds when $\Delta t\ll1$. In the limit $\Delta t\to0$, this converges to
\begin{align}
    d\V{x} = \sqrt{\frac{d\left[\sigma^{2}(t)\right]}{dt}}\ d\bm{\ell}(t), \label{eq:vesde non-Gaussian}
\end{align}
where $\bm{\ell}(t)$ is a L\'{e}vy process, rather than the Wiener process $\V{w}(t)$ of \citet{song2020score} which can be solved in closed-form as an affine Brownian motion SDE.

The addition of $\bm{\ell}$ to the more formal version of the diffusion process in \eqref{eq:diffusion process} gives
\begin{align}
    \V{x}(t) = \int_{0}^{t}\V{f}(\V{x},s)ds + \int_{0}^{t}g(s)d\bm{\ell},
\end{align}
where the latter term can be interpreted as
\begin{align}
    \lim_{\Delta t\to 0}\left[\sum\limits_{i}g(t_{i})(\bm{\ell}(t_{i}+\Delta t)-\bm{l}(t_{i}))\right].
\end{align}
The sum formulation makes it clear that, for any infinitely divisible distribution\footnote{$F$ is infinitely divisible if $\forall n\in\mathbb{N},\ \exists\ n$ i.i.d. RVs s.t. $\sum\limits_{i=1}^{n}X_{i}=S$ and $S$ has the same distribution as $F$.} which sums to itself (e.g. Gaussian, Laplace, and the previously discarded Cauchy), the final distribution will be in the same family.

Therefore, it is expected that the solution to \eqref{eq:vesde non-Gaussian} describes a process which would diffuse to the underlying stable (infinitely divisible) distribution. In the context of SMLD, this means that the prior $p_{T}(\V{x})$ need not be Gaussian and can be heavy tailed.

Unfortunately, to reverse the diffusion, it is necessary to investigate the general, non-Brownian, form of the Kolmogorov backward equations, in an analysis beyond that of \citet{song2020score} and this paper. Instead, several practical remarks are made to finish the theory of this work.

Firstly, it is of note that the generalised normal distribution considered in this chapter is infinitely divisible for $\beta\in(0,1]\cup \{2\}$ \cite{dytso2018analytical}. This result is interesting because the analysis in Section~\ref{sec:Comparison of concentration moments} suggests that $\beta<1$ suffers from explosive and unwieldy norm distribution moment coefficients, yet this region may be of theoretical intrigue for continuous HTDSM. Secondly, there exist several connections between Brownian motion and heavier-tailed diffusion through subordination---letting time evolve according to a stochastic process within another stochastic process. The prime examples of this are variance gamma (VG) processes, which can be written as a Brownian motion $W(t)$ with drift $\theta t$, subject to a random time change that follows a gamma process $\Gamma(t;1,\nu)$
\begin{align}
    X^{VG}(t;\sigma,\nu,\theta) = \theta\Gamma(t;1,\nu) + \sigma W(\Gamma(t;1,\nu)), \label{eq:variance gamma}
\end{align}
where $\sigma$ is a scale parameter and $\nu$ controls the time dilation. In particular, when $\nu=1$, a VG process is equivalent to the continuous version of the $\beta=1$ GN noise considered in this chapter. Future work may be able to use the backward Kolmogorov equation on the time-dilated Wiener process to form an ordinary differential equation describing the reverse evolution from a heavy-tailed prior to the data distribution.

\subsection{Proof of Theorem~\ref{thm:piecewise differentiable}}\label{sec:Proofs}

The relevant assumptions from \cite{hyvarinen2005estimation} are:
\begin{itemize}
    \item[\textbf{\#1}] The PDF $p(\V{x})$ is differentiable.
    \item[\textbf{\#2}] $\mathbb{E}_{p(x)}\left[\norm[\bigg]{\frac{\partial\log p(\V{x})}{\partial\V{x}}}^{2}\right]$ is finite.
    \item[\textbf{\#3}] For any $\theta$:
    \begin{itemize}
        \item[\textbf{A}] $\mathbb{E}_{p(\V{x})}\left[\lVert s_{\theta}(\V{x})\rVert^{2}\right]$ is finite.
        \item[\textbf{B}] $\lim\limits_{\lVert\V{x}\rVert\to\infty}\left[p(\V{x})s_{\theta}(\V{x})\right]=0$.
    \end{itemize}
\end{itemize}

\begin{proof}
Expanding \eqref{eq:esm} gives
\begin{align}
    \mathcal{J}_{ESMp}(\theta) &= \int_{\V{x}\in\mathbb{R}^{n}} p(\V{x}) \left[\frac{1}{2}\underbrace{\lVert\nabla_{\V{x}}\log p(\V{x})\rVert_{2}^{2}}_{\textrm{\ding{172}}} + \frac{1}{2}\underbrace{\lVert s_{\theta}(\V{x})\rVert_{2}^{2}}_{\textrm{\ding{173}}} - \underbrace{\nabla_{\V{x}}\log p(\V{x})^{T}s_{\theta}(\V{x})}_{\textrm{\ding{174}}} \right]d\V{x},
\end{align}
where \ding{172} can be ignored as it is constant, with no dependency on $\theta$. For the second term, expand
\begin{align}
    \textrm{\ding{173}} = \int_{\V{x}\in\mathbb{R}^{n}} p(\V{x}) \sum\limits_{i=1}^{n}(s_{\theta}(\V{x})_{i})^{2} d\V{x},
\end{align}
where $s_{\theta}(\V{x})_{i}$ is the $i^{\textrm{th}}$ component of the partial derivatives composing $s_{\theta}(\V{x})$, and let 
\begin{align}
    s_{\theta}(\V{x})_{i} = \sum\limits_{j} s_{\theta}(\V{x})_{i,j},
\end{align}
where $j$ indexes a countable sequence of intervals partitioning the real line (except for points of zero measure). Also let each $s_{\theta}(\V{x})_{i,j}$ be differentiable inside its corresponding interval and zero outside, permitting the derivation of
\begin{align}
    \int_{\V{x}\in\mathbb{R}^{n}} p(\V{x}) \sum\limits_{i=1}^{n}(s_{\theta}(\V{x})_{i})^{2} d\V{x} &= \sum\limits_{i=1}^{n} \int_{\V{x}\in\mathbb{R}^{n}} p(\V{x}) \left(\sum\limits_{j}s_{\theta}(\V{x})_{i,j}\right)^{2} d\V{x} \\
    &= \sum\limits_{i=1}^{n} \int_{\V{x}\in\mathbb{R}^{n}} p(\V{x}) \sum\limits_{j}\left(s_{\theta}(\V{x})_{i,j}\right)^{2} d\V{x} \\
    &= \int_{\V{x}\in\mathbb{R}^{n}} p(\V{x}) \sum\limits_{i=1}^{n}\left(s_{\theta}(\V{x})_{i}\right)^{2} d\V{x}\\
    &= \int_{\V{x}\in\mathbb{R}^{n}} p(\V{x}) \lVert s_{\theta}(\V{x})\rVert_{2}^{2} d\V{x},
\end{align}
the first term of \eqref{eq:ism}, despite the differentiable almost everywhere formulation. This step of the proof simply shows that \textit{when integrating, the square of the sum of piecewise non-zero functions is equal to the sum of their squares}.\\

\ding{174} remains, and the proof will be complete if a differentiable almost everywhere equivalent of Lemma 4 in \cite{hyvarinen2005estimation} establishes a multivariate version of
\begin{align}
    \int p(x)(\log p)'f(x)dx = \int p(x)\frac{p'(x)}{p(x)}f(x)dx = \int p'(x)f(x)dx = \int p(x)f'(x)dx.
\end{align}

\begin{proposition}
    For $i=1$, without loss of generality (WLOG)
    \begin{align}
        \lim\limits_{a\to\infty,\ b\to-\infty}[f(a,x_{2},\ldots,x_{n})g(a,x_{2},\ldots,x_{n}) &- f(b,x_{2},\ldots,x_{n})g(b,x_{2},\ldots,x_{n})] \nonumber \\
        &= \int_{-\infty}^{\infty}f(x)\frac{\partial g(x)}{\partial x_{1}}d x_{1} + \int_{-\infty}^{\infty}g(x)\frac{\partial f(x)}{\partial x_{1}}d x_{1},
    \end{align}
    assuming that $f$ is differentiable and $g$ is differentiable \underline{almost everywhere}. 
\end{proposition}
\begin{proof}
WLOG break $g(x)$ into piecewise differentiable and non-zero functions along the first dimension, $g(x)=\sum\limits_{j}g_{j}(x)$, defined in the interval $I_{j}$ and zero elsewhere. Then
\begin{align}
    \frac{\partial f(x)g(x)}{\partial x_{1}} &= f(x)\frac{\partial g(x)}{\partial x_{1}} + g(x)\frac{\partial f(x)}{\partial x_{1}} \\
    &= f(x) \frac{\partial}{\partial x_{1}}\left[\sum\limits_{j} g_{j}(x)\right] + \sum\limits_{j} g_{j}(x)\frac{\partial f(x)}{\partial x_{1}}, 
\end{align}
where all variables except $x_{1}$ can be fixed. Then, integrating over $x_{1}\in\mathbb{R}$,
\begin{align}
    [f(x)g(x)]_{-\infty}^{\infty} &= \int_{-\infty}^{\infty}f(x)\sum\limits_{j}\frac{\partial g_{j}(x)}{\partial x_{1}}dx_{1} + \sum\limits_{j} \int_{I_{j}} g_{j}(x)\frac{\partial f(x)}{\partial x_{1}}dx_{1} \\
    &= \int_{-\infty}^{\infty}f(x) \frac{\partial g(x)}{\partial x_{1}} dx_{1} + \int_{-\infty}^{\infty}g(x)\frac{\partial f(x)}{\partial x_{1}}dx_{1},
\end{align}
where the first term arises by construction and the second arises via a telescoping sum.
\end{proof}

This proposition allows for an equivalent to the final step in \cite{hyvarinen2005estimation}
\begin{multline}
    -\int \frac{\partial p_{\V{x}}(\V{x})}{\partial x_{1}}s_{\theta}(\V{x})d\V{x} = -\int \left[\int \frac{\partial p_{\V{x}}(\V{x})}{x_{1}}s_{\theta}(\V{x})dx_{1} \right]d(x_{2},\ldots,x_{n}) \\
    = -\int \bigg[\lim_{a\to\infty,b\to-\infty}[p_{\V{x}}(a,x_{2}\,\ldots,x_{n})s_{\theta}(a,x_{2}\,\ldots,x_{n}) \nonumber \\
    - p_{\V{x}}(b,x_{2}\,\ldots,x_{n})s_{\theta}(b,x_{2}\,\ldots,x_{n})] \nonumber\\
    - \int\frac{s_{\theta}(\V{x})}{\partial x_{1}} p_{\V{x}}(\V{x}) dx_{1} \bigg]d(x_{2},\ldots,x_{n}).
\end{multline}
The choice of $i=1$ is arbitrary and the limit is zero by assumption, therefore proving
\begin{align}
    -\int_{-\infty}^{\infty} p_{\V{X}}(\V{x})\frac{\partial\log p_{\V{X}}(\V{x})}{\partial x_{i}}s_{\theta}(\V{x})_{i}dx_{i} = \int \frac{\partial s_{\theta}(\V{x})_{i}}{\partial x_{i}} p_{\V{x}}(\V{x})dx_{i},
\end{align}
returns the $i^{\textrm{th}}$ component \ding{174}, which is summed to form the trace.
\end{proof}

\subsection{Derivation figures}\label{sec:Additional figures}

\begin{figure}[H]
    \centering
    \includegraphics[width=0.3\linewidth]{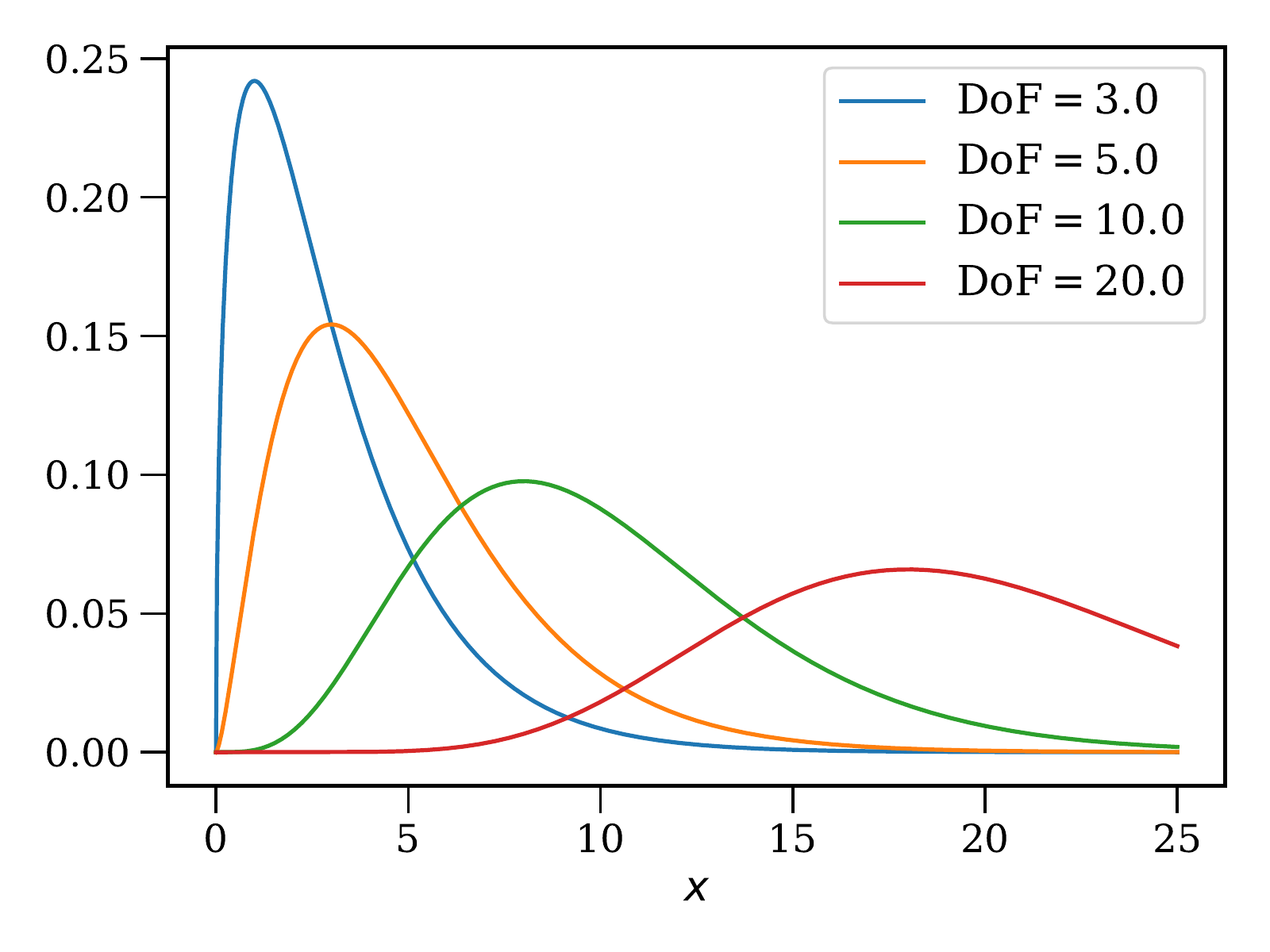}
    \caption{The chi-squared distribution for different degrees of freedom (DoF).}
    \vspace{-10pt}
    \label{fig:chi-squared}
\end{figure}

\begin{figure}[H]
    \centering
    \subfigure[Linear scale.]{\includegraphics[width=0.3\linewidth]{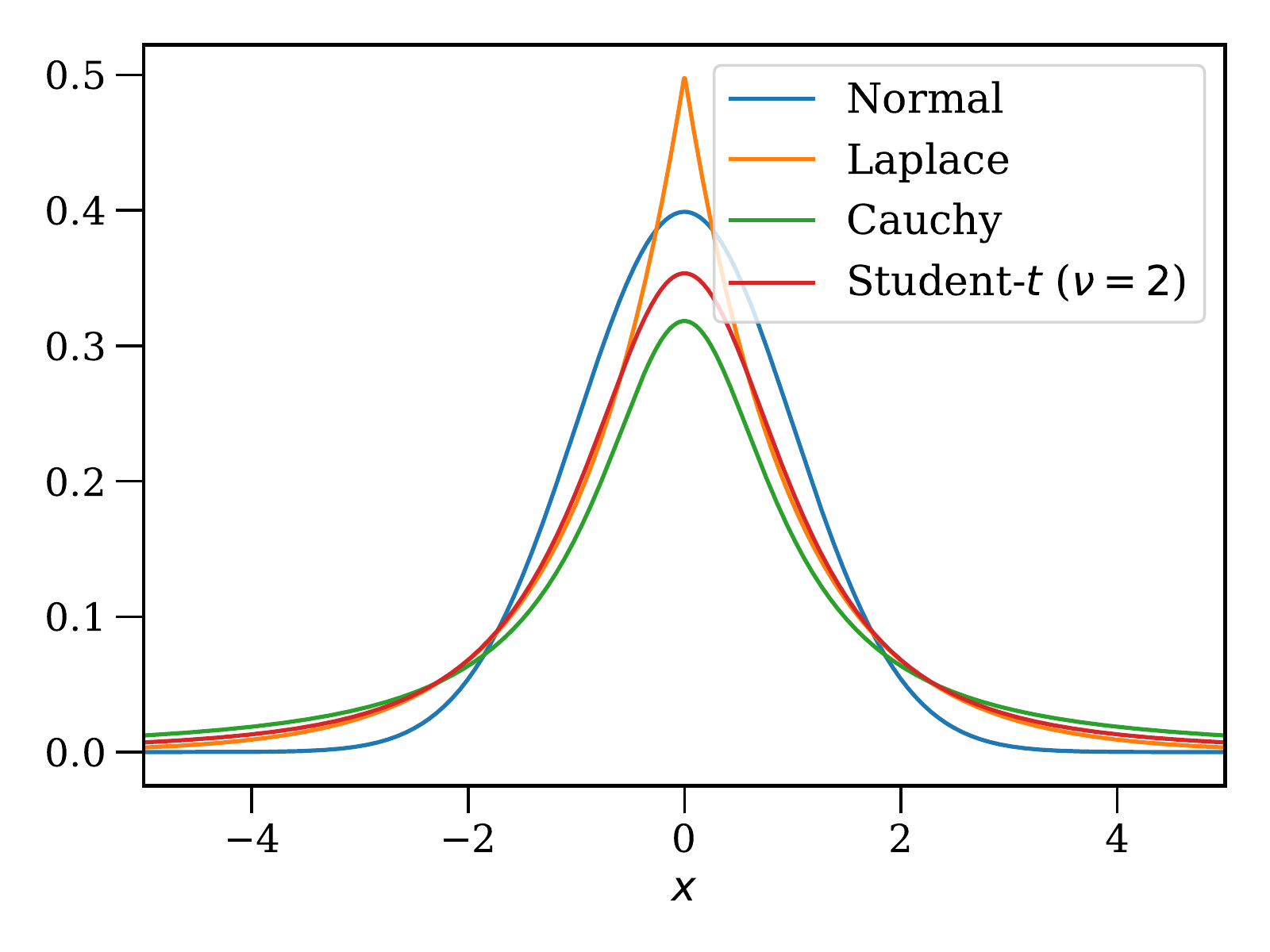}}\hspace{50pt}%
    \subfigure[Log scale.]{\includegraphics[width=0.3\linewidth]{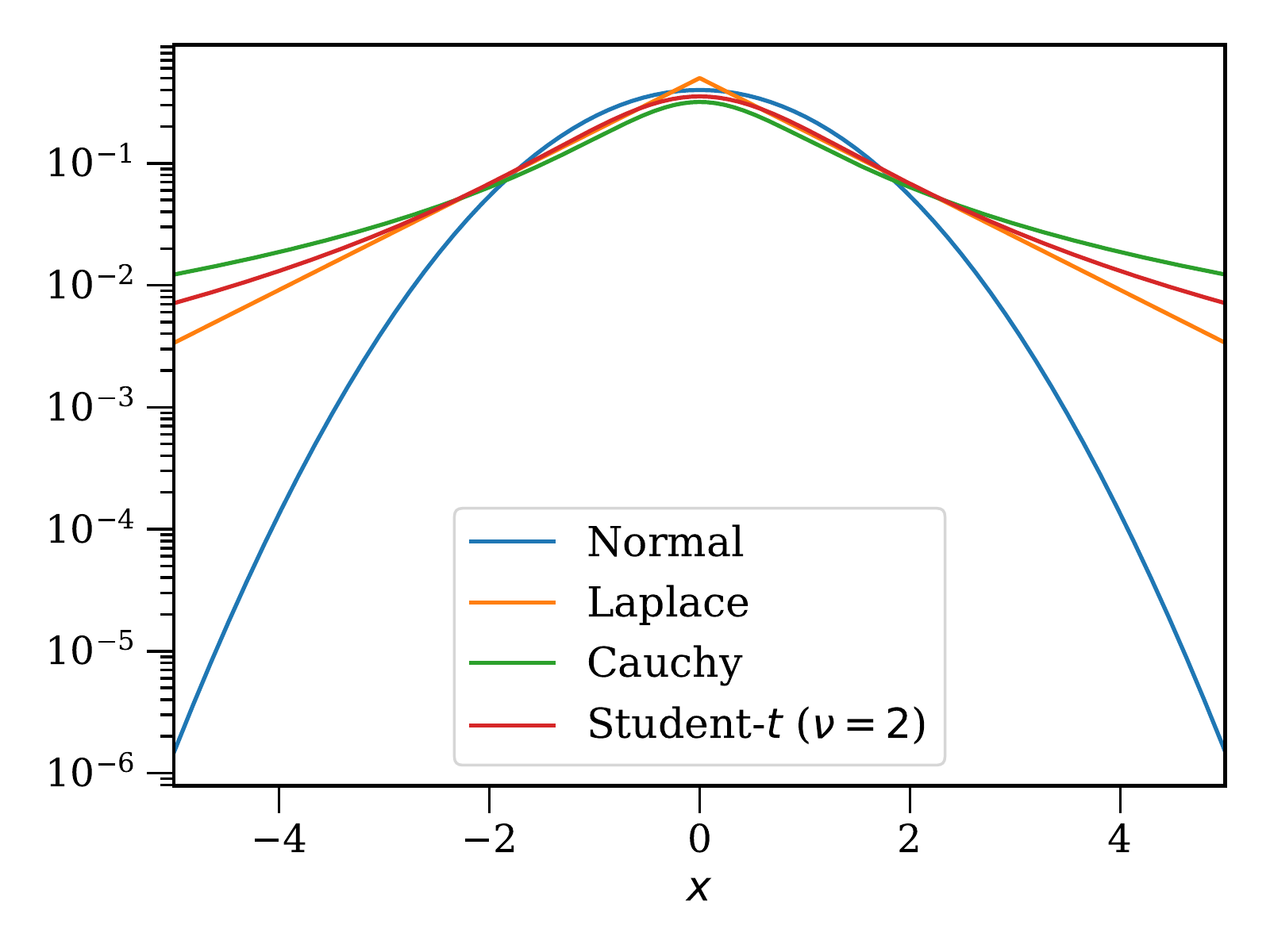}}
    \vspace{-10pt}
    \caption{Comparison of the Gaussian distribution and common heavy tailed distributions.}
    \label{fig:heavy_tailed_distributions}
\end{figure}


\begin{figure}[H]
    \centering
    \subfigure[Mean scaling factor from \eqref{eq:mean_scaling}.]{\includegraphics[width=0.3\linewidth]{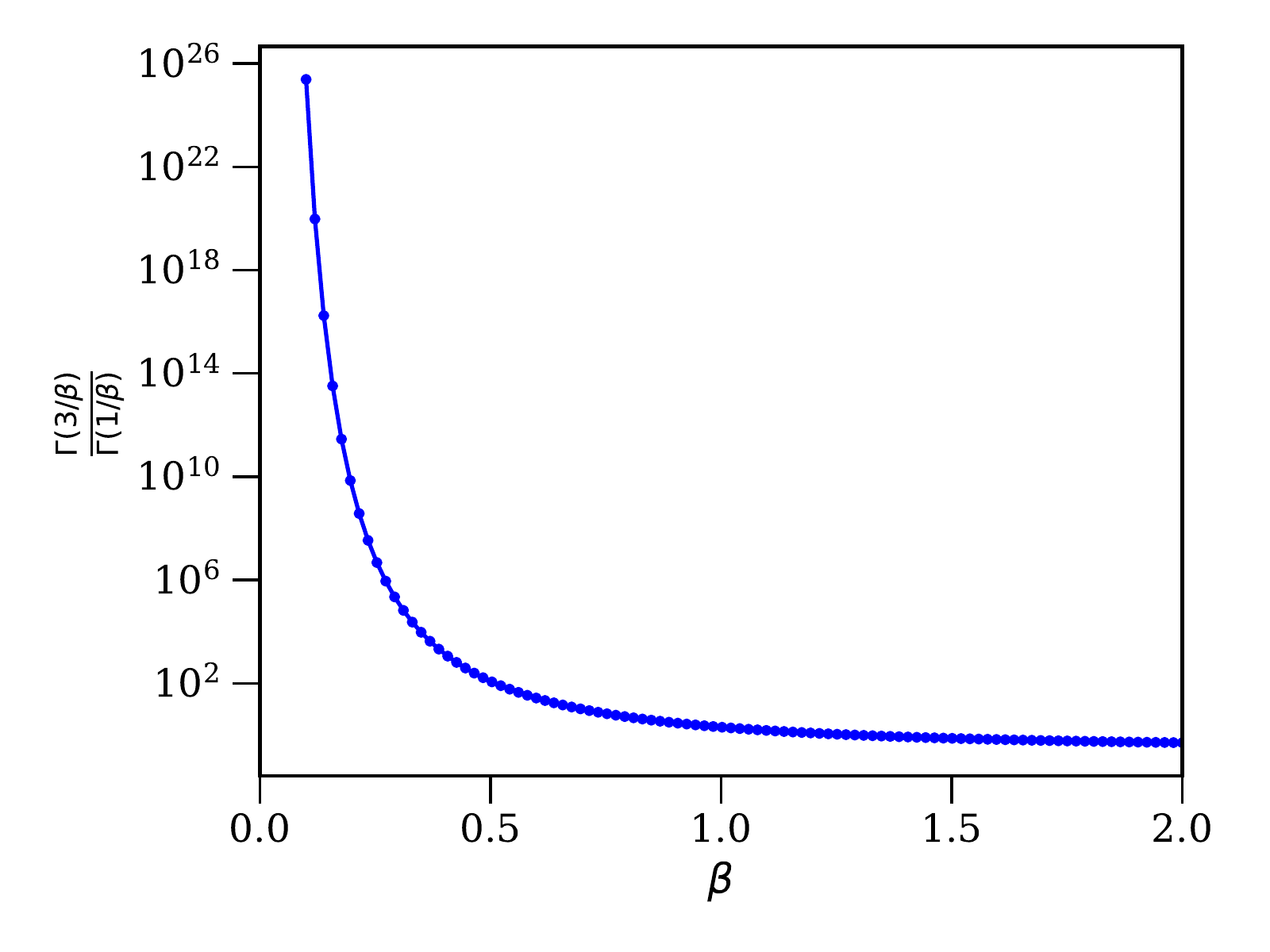}\label{fig:gamma_trends2}}%
    \subfigure[Variance scaling factor from \eqref{eq:var_scaling}.]{\includegraphics[width=0.3\linewidth]{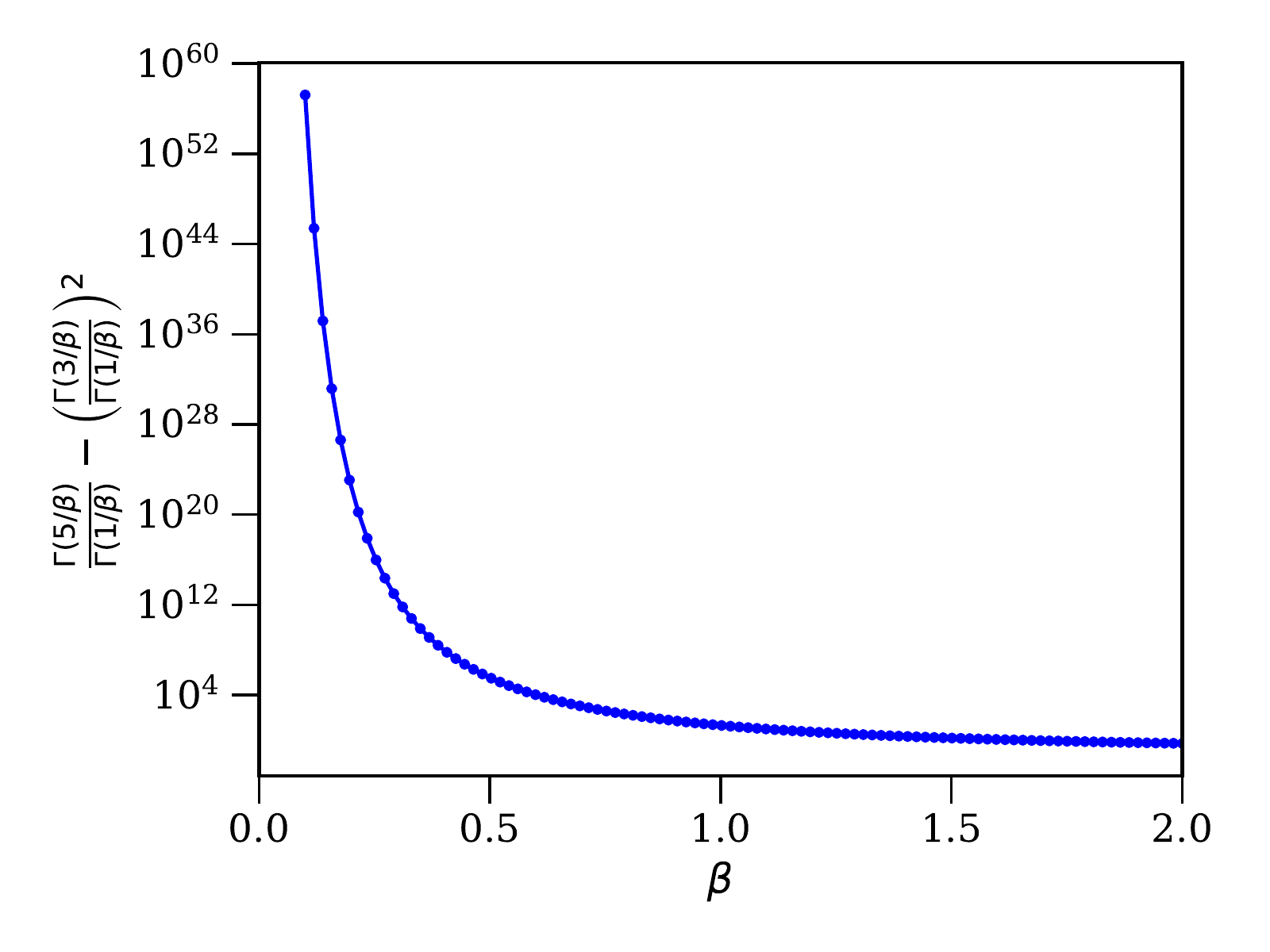}\label{fig:gamma_trends3}}
    \vspace{-10pt}
    \caption{Scaling factor trends against $\beta$.}
\end{figure}

\clearpage

\section{Extended results}\label{sec:2d_rez}
\subsection{Extended $2D$ example}
\begin{figure}[H]
    {\centering
    \subfigure[]{\includegraphics[width=0.45\linewidth]{Chapter4/Figs/1_data.png}} \hfill%
    \subfigure[]{\label{fig:dsm_2}\includegraphics[width=0.45\linewidth]{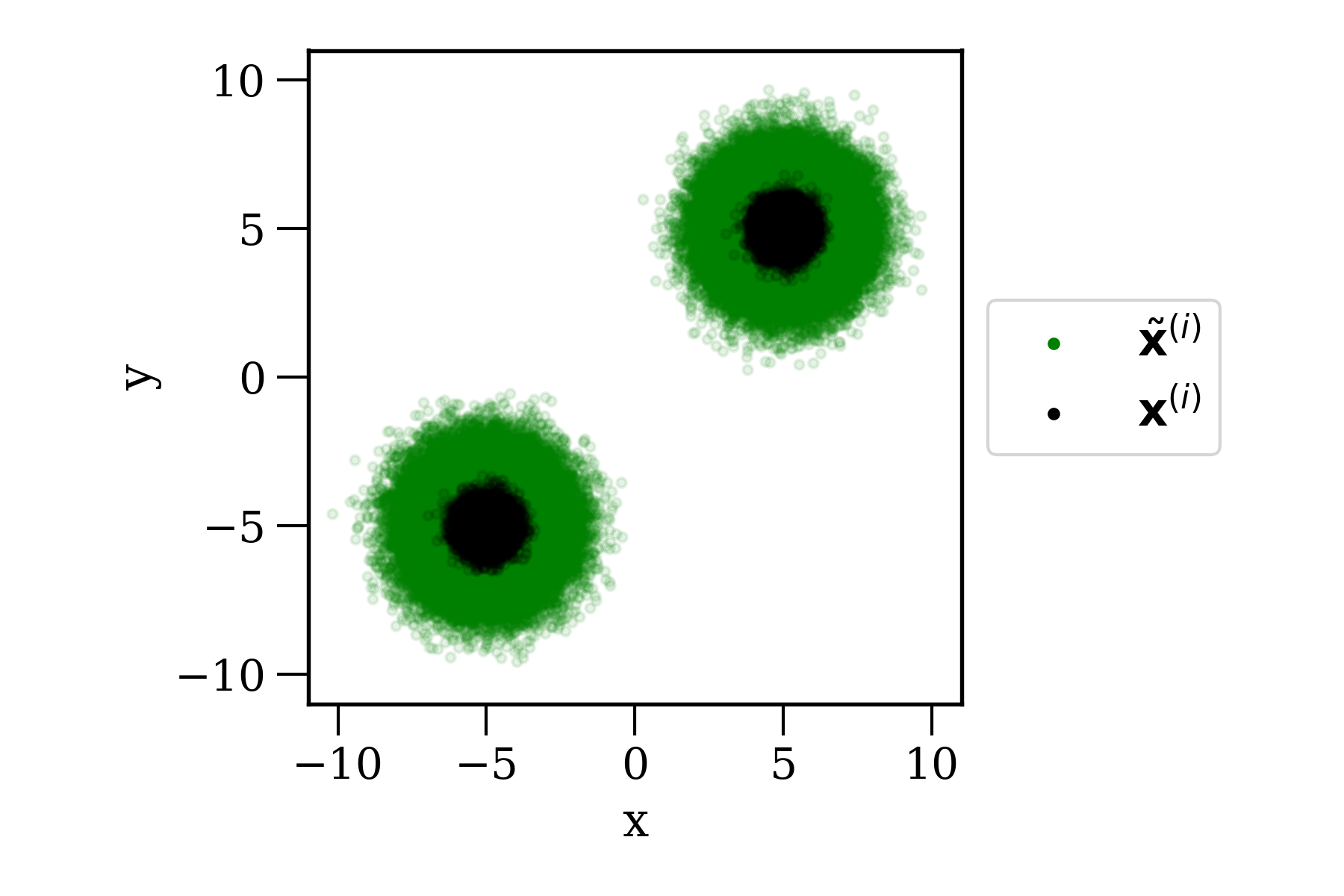}}\vspace{-10pt}
    \subfigure[]{\label{fig:dsm_3}\includegraphics[width=0.45\linewidth]{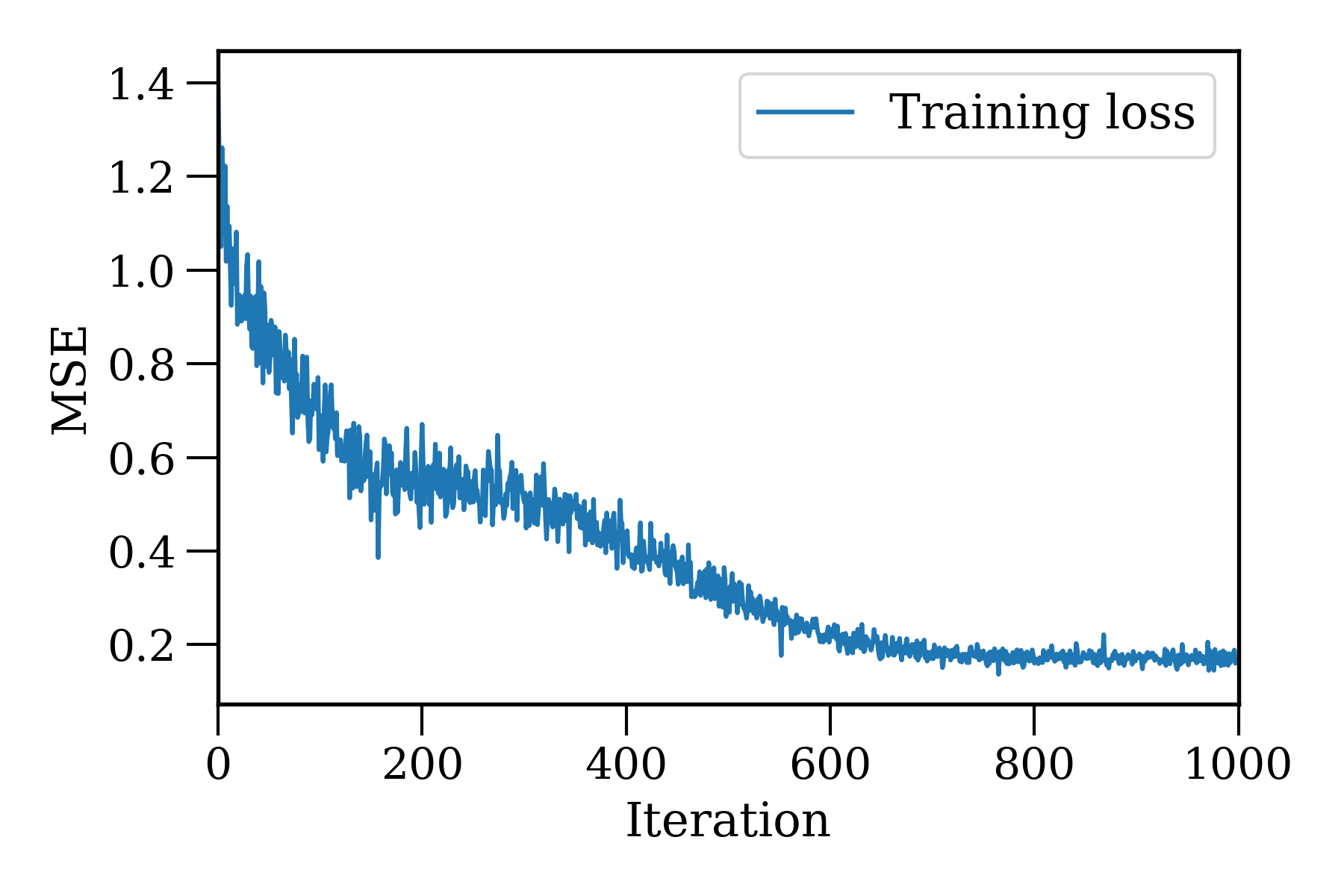}} \hfill%
    \subfigure[]{\label{fig:dsm_4}\includegraphics[width=0.45\linewidth]{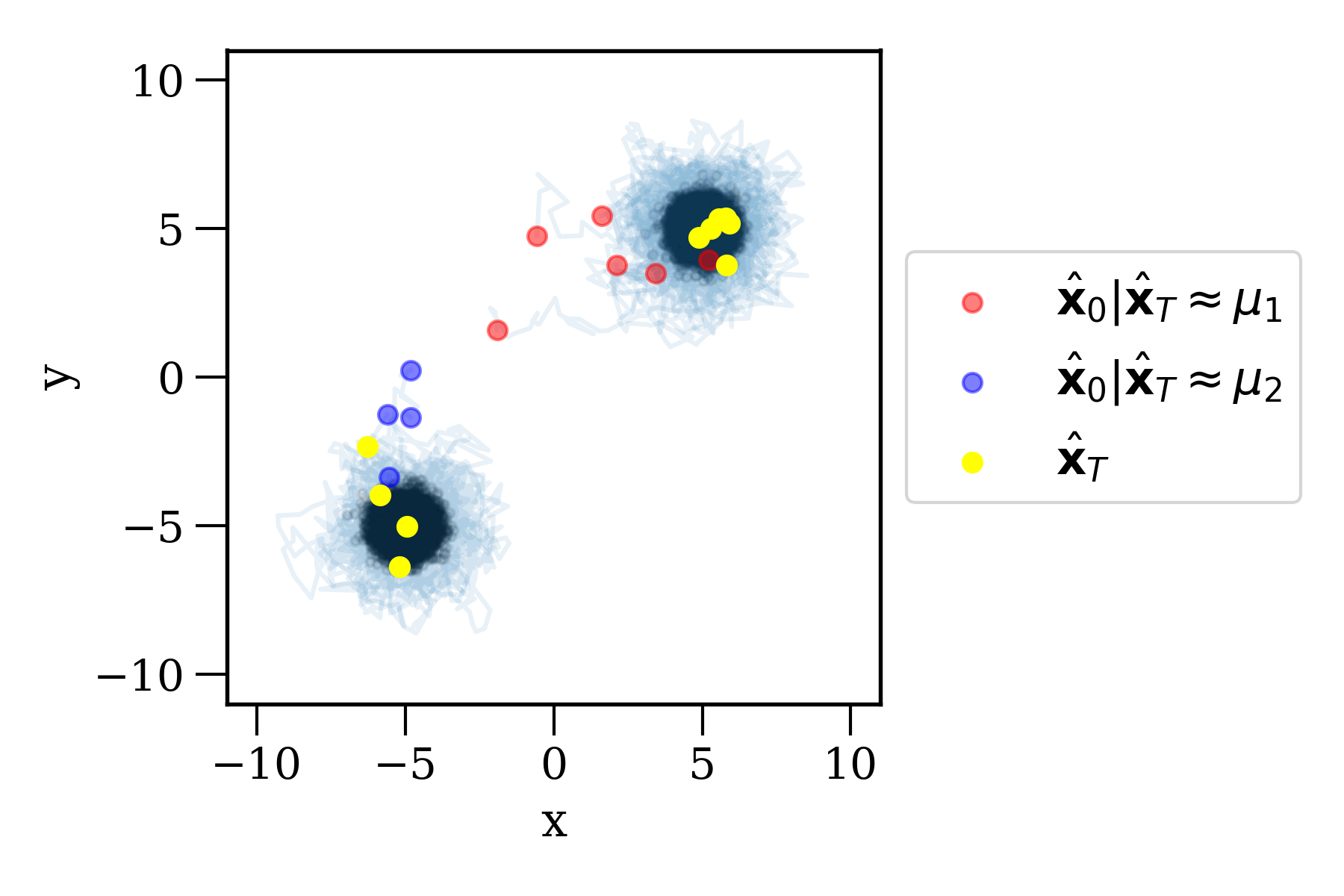}}\vspace{-10pt}
    \subfigure[]{\label{fig:dsm_5}\includegraphics[width=0.45\linewidth]{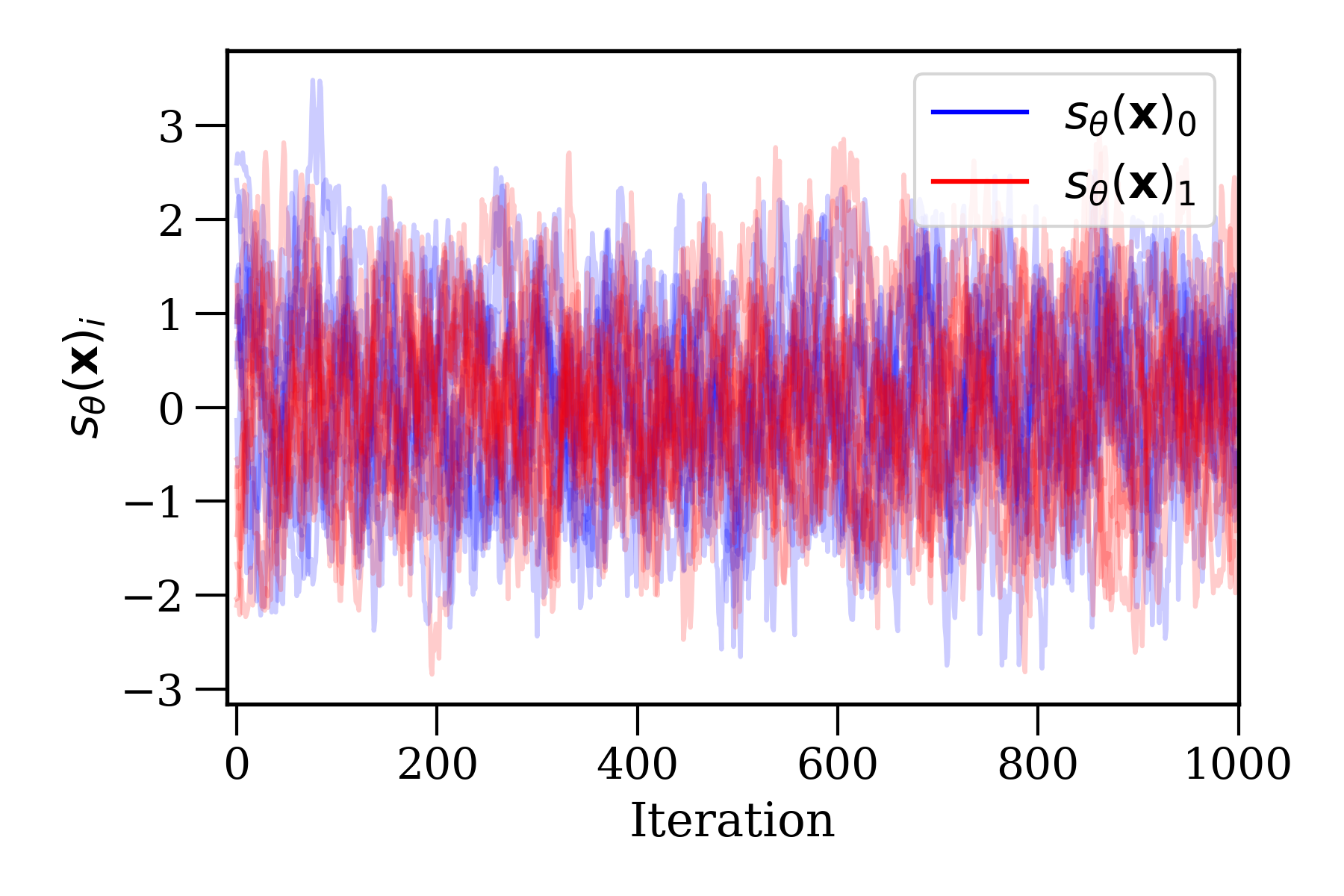}} \hfill%
    \subfigure[]{\includegraphics[width=0.45\linewidth]{Chapter4/Figs/6_more_paths.png}}}\vspace{-10pt}
    \caption{DSM training and LD sampling. In \textbf{a}, $p(\V{x})$ is modelled as an additive mixture of ($k=2$) bivariate Gaussians with 10,000 samples per mode. A depth 3 MLP ($2\to16\to16\to2$, intermediate activations ReLU, batch size 256) is trained to estimate the score from samples noised by $q_{\sigma}(\Tilde{\V{x}}|\V{x})\sim\mathcal{N}(\V{x},\V{I})$. All training noise samples are shown as green in \textbf{b} and the training convergence is depicted in \textbf{c}. Then, in \textbf{d}, starting from $\hat{\V{x}}_{0}\sim\mathcal{U}(-6,6)\times\mathcal{U}(-6,6)$, 10 sampled particles are evolved to convergence using 1,000 steps of Langevin Dynamics with step size 0.1 and matching noise scale. The score estimates used during sampling are presented in \textbf{e}. Finally, the same sampling is repeated in \textbf{f} for 1,000 particles to demonstrate the decision boundary, its asymmetry (relevant for class imbalance), and the upper bound on approximation accuracy due to the underlying unit noise.}
    \label{fig:dsm example 2}
\end{figure}

\begin{figure}
    {\centering
    \subfigure[]{\includegraphics[width=0.45\linewidth]{Chapter4/Figs/2_noised_data_2level.png}} \hfill%
    \subfigure[]{\label{fig:dsm_ald_2}\includegraphics[width=0.45\linewidth]{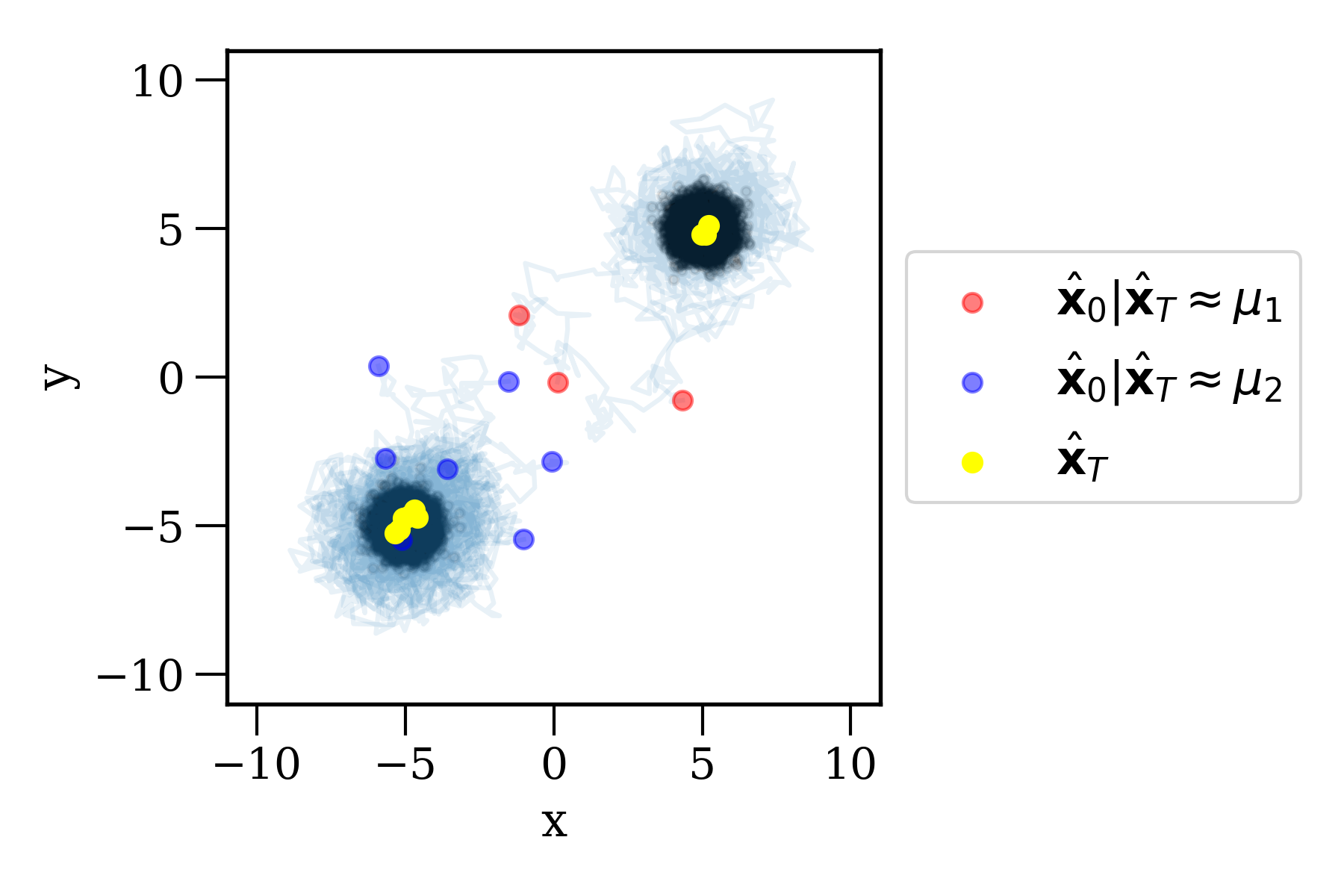}}
    \subfigure[]{\label{fig:dsm_ald_3}\includegraphics[width=0.45\linewidth]{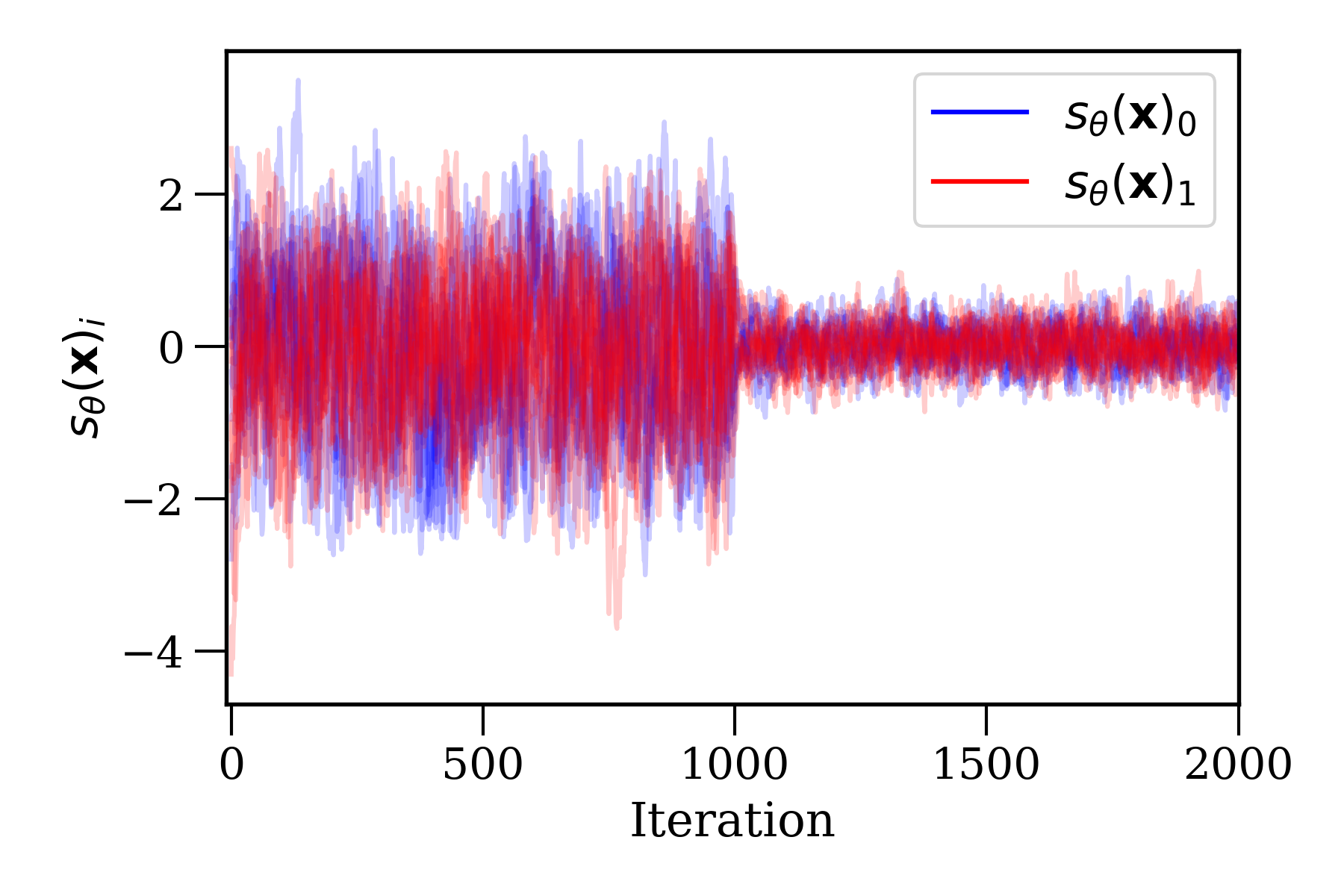}} \hfill%
    \subfigure[]{\includegraphics[width=0.45\linewidth]{Chapter4/Figs/6_more_paths_2level.png}}}
    \caption{Multiple noise level DSM training and ALD sampling. The setup and figures are identical to Figure~\ref{fig:dsm example 2} except that two noise scales, $\sigma_{1}=1.0$ and $\sigma_{2}=0.25$, are used in training and sampling. The sequential use of decreasing noise levels in sampling can be seen in \textbf{c}. It is evident that ALD drastically improves the final distribution estimate due to the decrease in score estimate scale. It is also relevant to subsequent class imbalance problems that the sampling procedure is slightly asymmetric. For all models trained, class asymmetry is consistent across sampling runs, but not across DSM retraining, so is an artefact of the model.}
    \label{fig:dsm ald example 2}
\end{figure}

\begin{figure}
    {\centering
    \subfigure[]{\label{fig:dsm_ald_laplace_1}\includegraphics[width=0.45\linewidth]{Chapter4/Figs/2_noised_data_2level_laplace.png}}%
    \subfigure[]{\label{fig:dsm_ald_laplace_2}\includegraphics[width=0.45\linewidth]{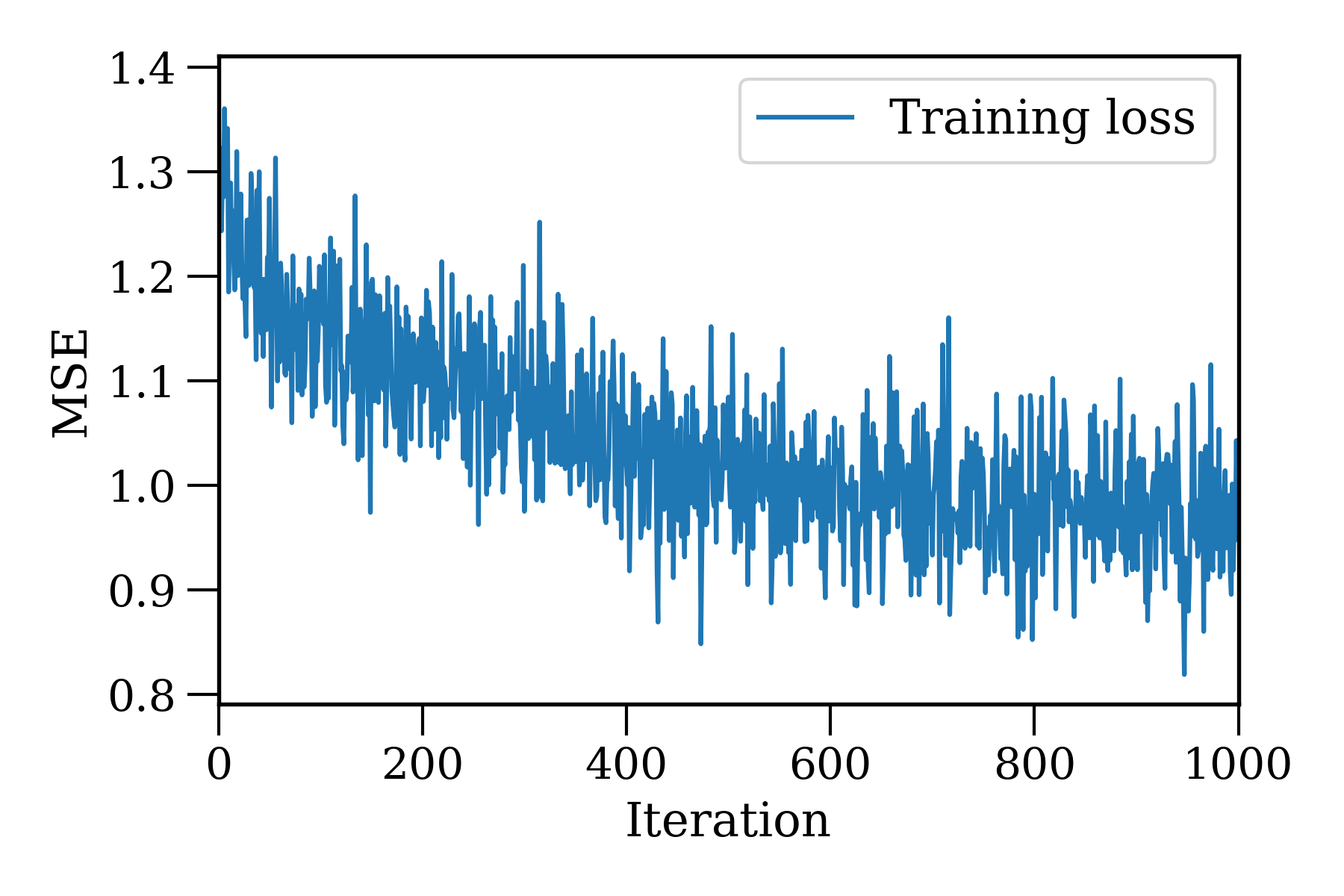}}
    \subfigure[]{\label{fig:dsm_ald_laplace_3}\includegraphics[width=0.45\linewidth]{Chapter4/Figs/4_paths_2level_laplace.png}}%
    \subfigure[]{\label{fig:dsm_ald_laplace_4}\includegraphics[width=0.45\linewidth]{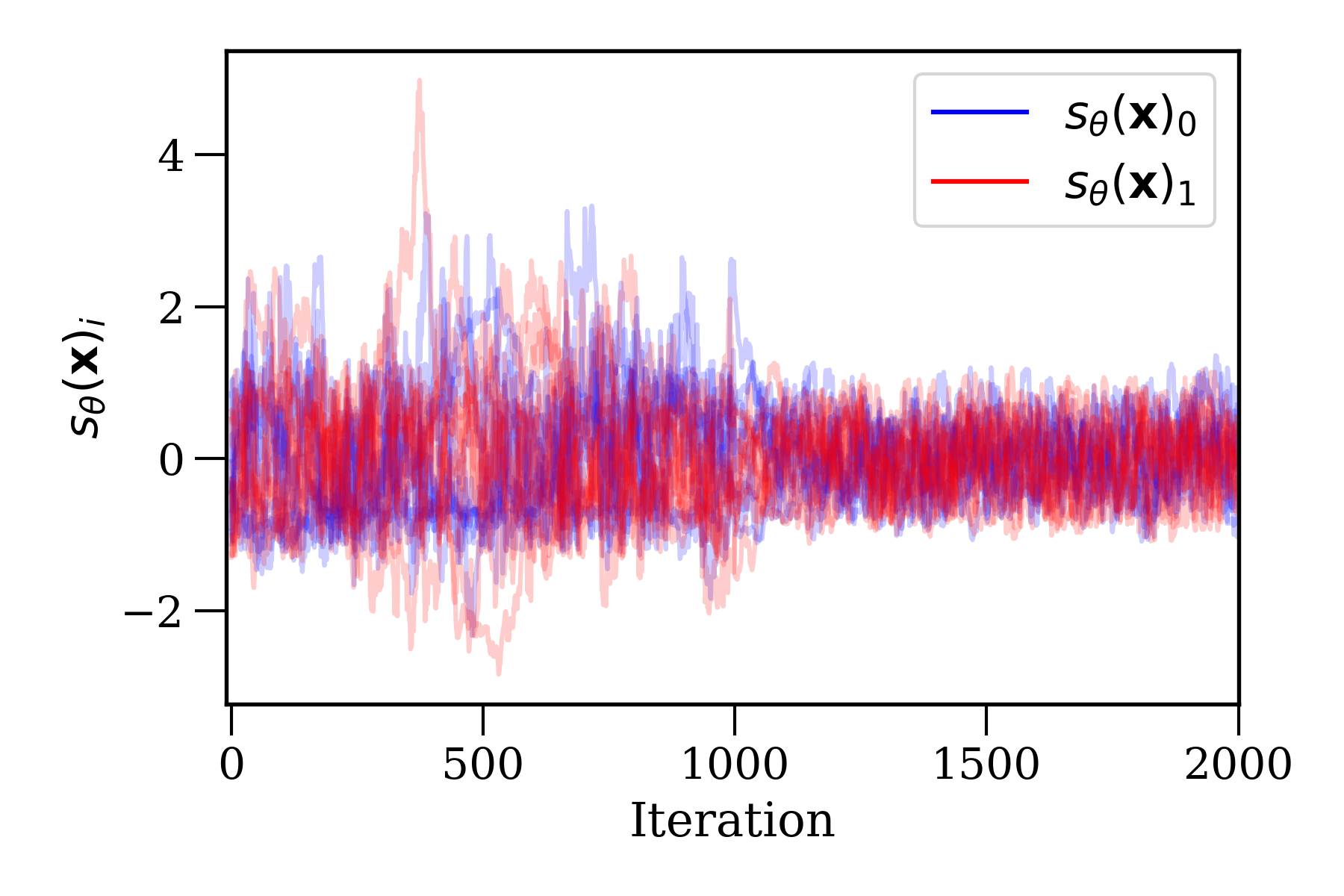}}
    \subfigure[]{\label{fig:dsm_ald_laplace_5}\includegraphics[width=0.45\linewidth]{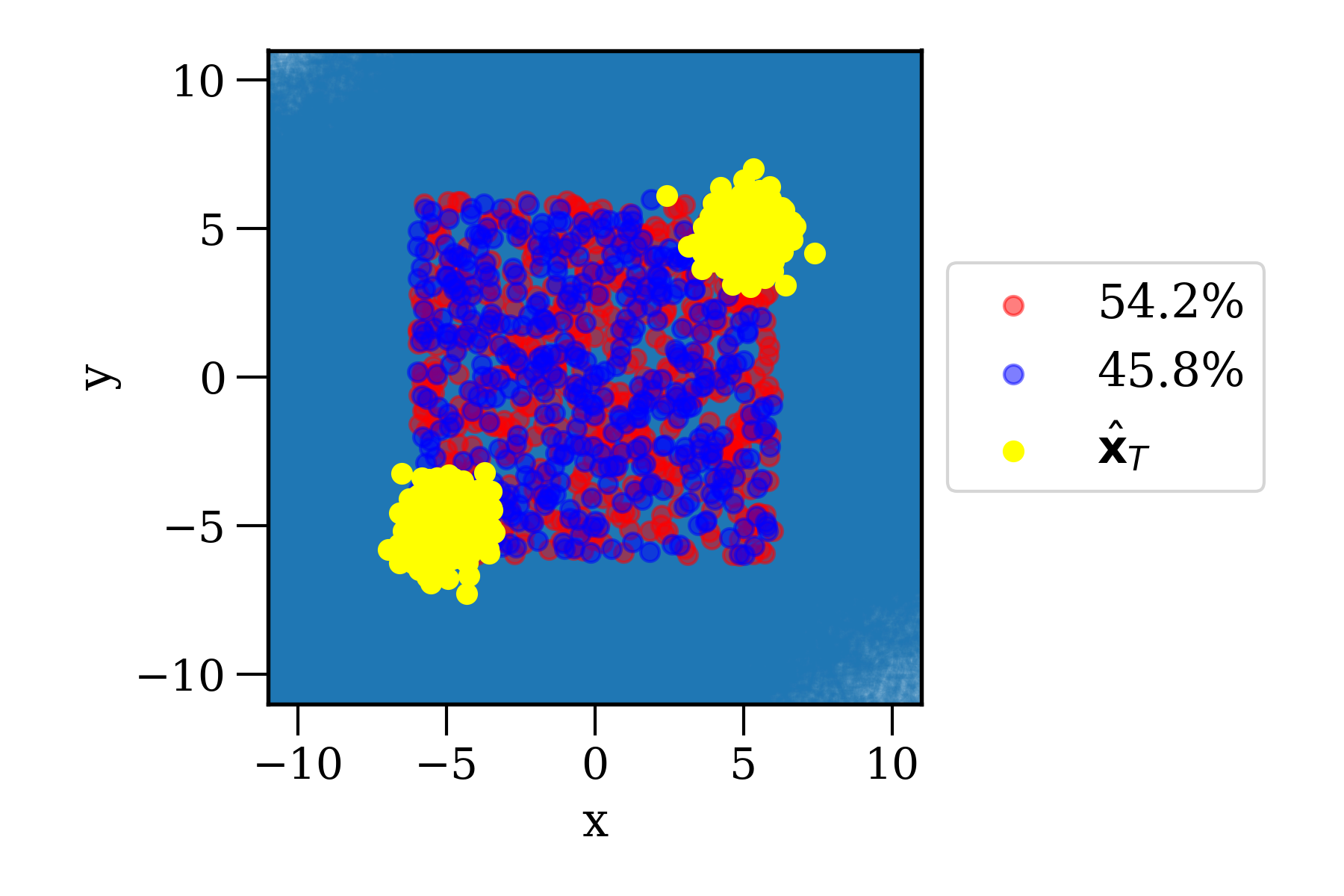}}}
    \caption{Laplace DSM with ALD. The setup and figures are identical to Figure~\ref{fig:dsm ald example 2}, except that Laplace noise is used ($\beta=1$ in the general formulation). \textbf{a} depicts the diamond, rather than circular, noise structure of a diagonal bivariate Laplace distribution. \textbf{b} and \textbf{c} respectively demonstrate that ALD training and sampling converge even with Laplace (sub-Gaussian, piece-wise differentiable) diffusion, confirming Theorem~\ref{thm:piecewise differentiable}.}
    \label{fig:dsm ald laplace 2}
\end{figure}

\begin{figure}[ht]
    {\centering
    \subfigure[]{\label{fig:dsm_ald_10x_1}\includegraphics[width=0.45\linewidth]{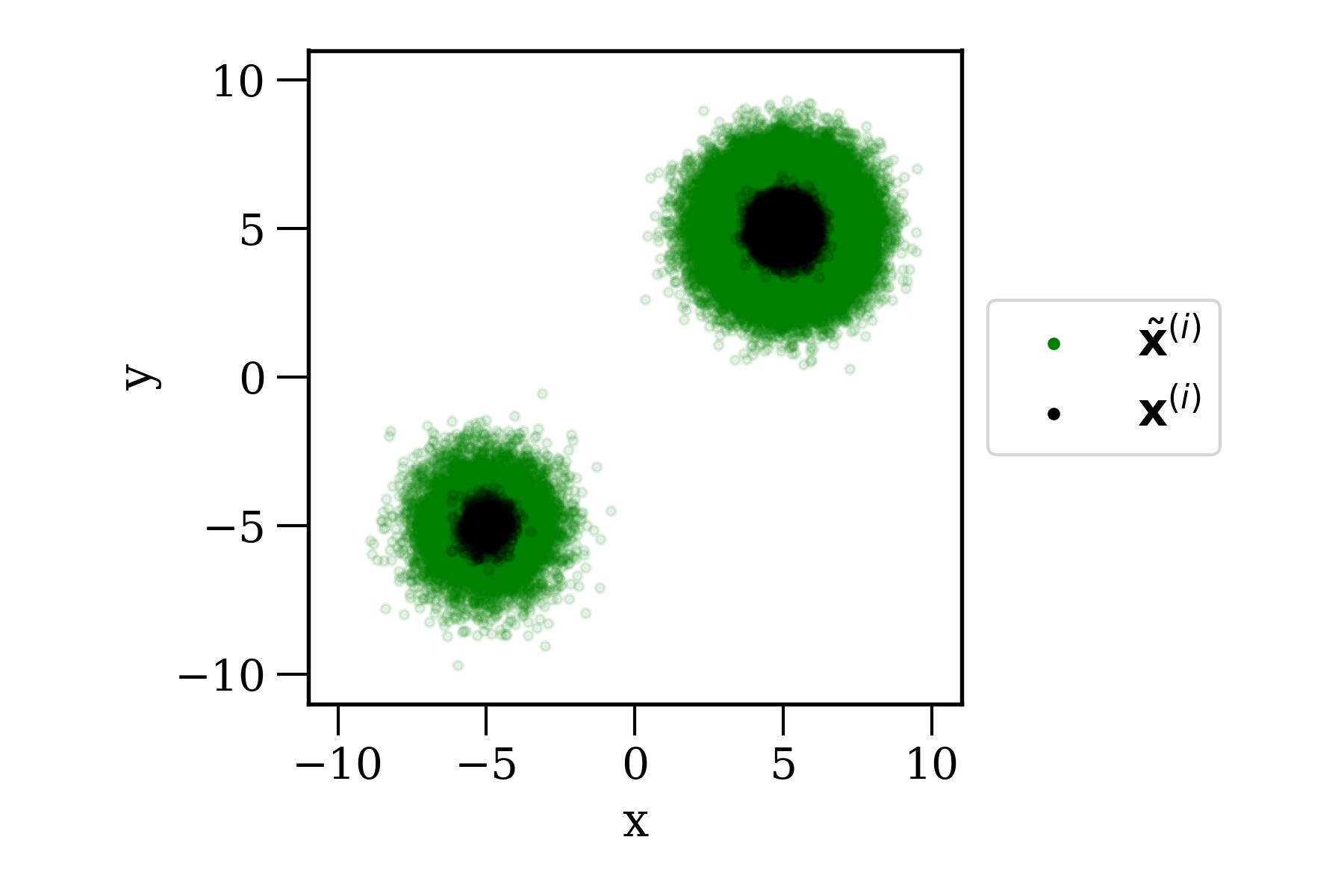}}%
    \subfigure[]{\label{fig:dsm_ald_10x_2}\includegraphics[width=0.45\linewidth]{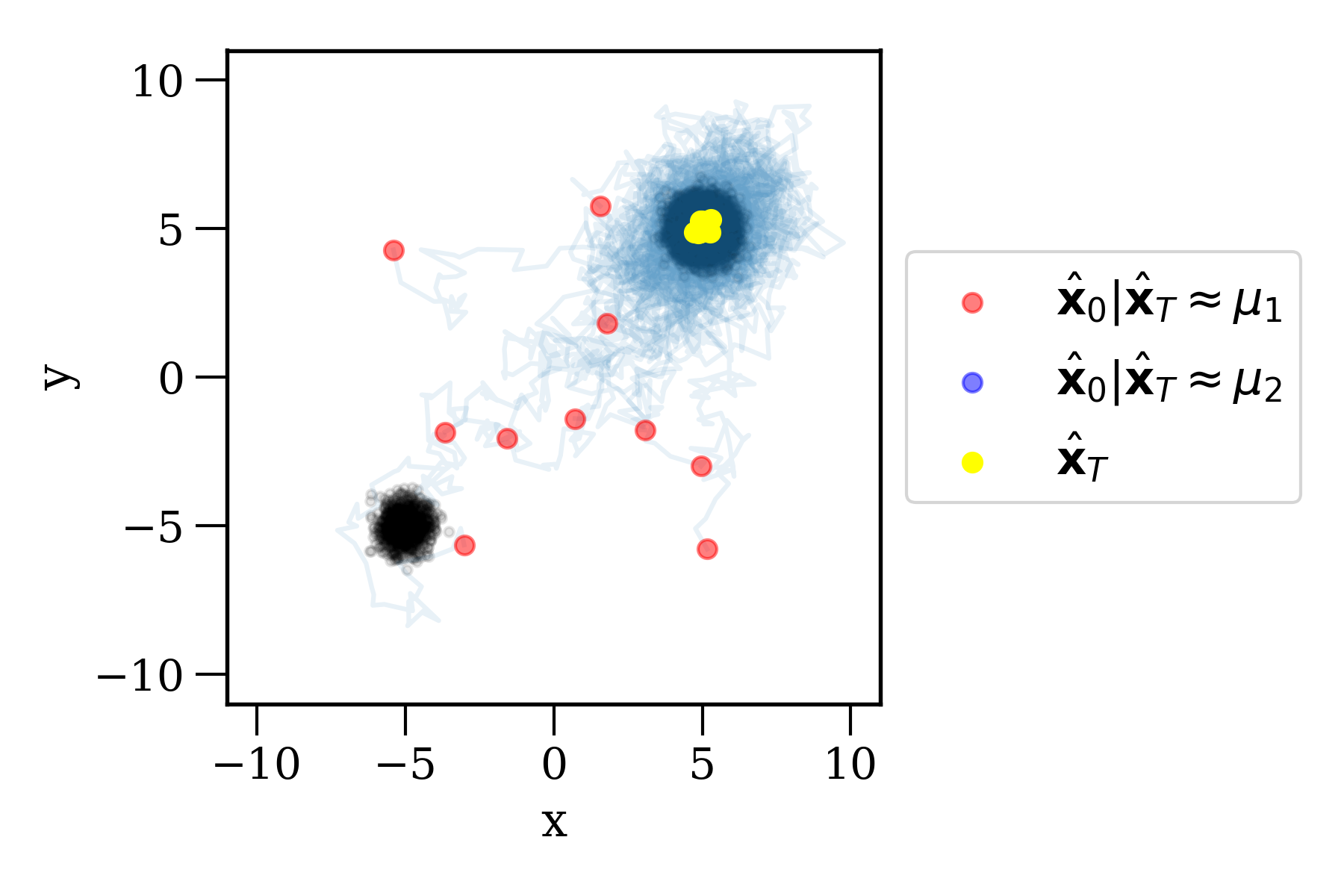}}
    \subfigure[]{\label{fig:dsm_ald_10x_3}\includegraphics[width=0.45\linewidth]{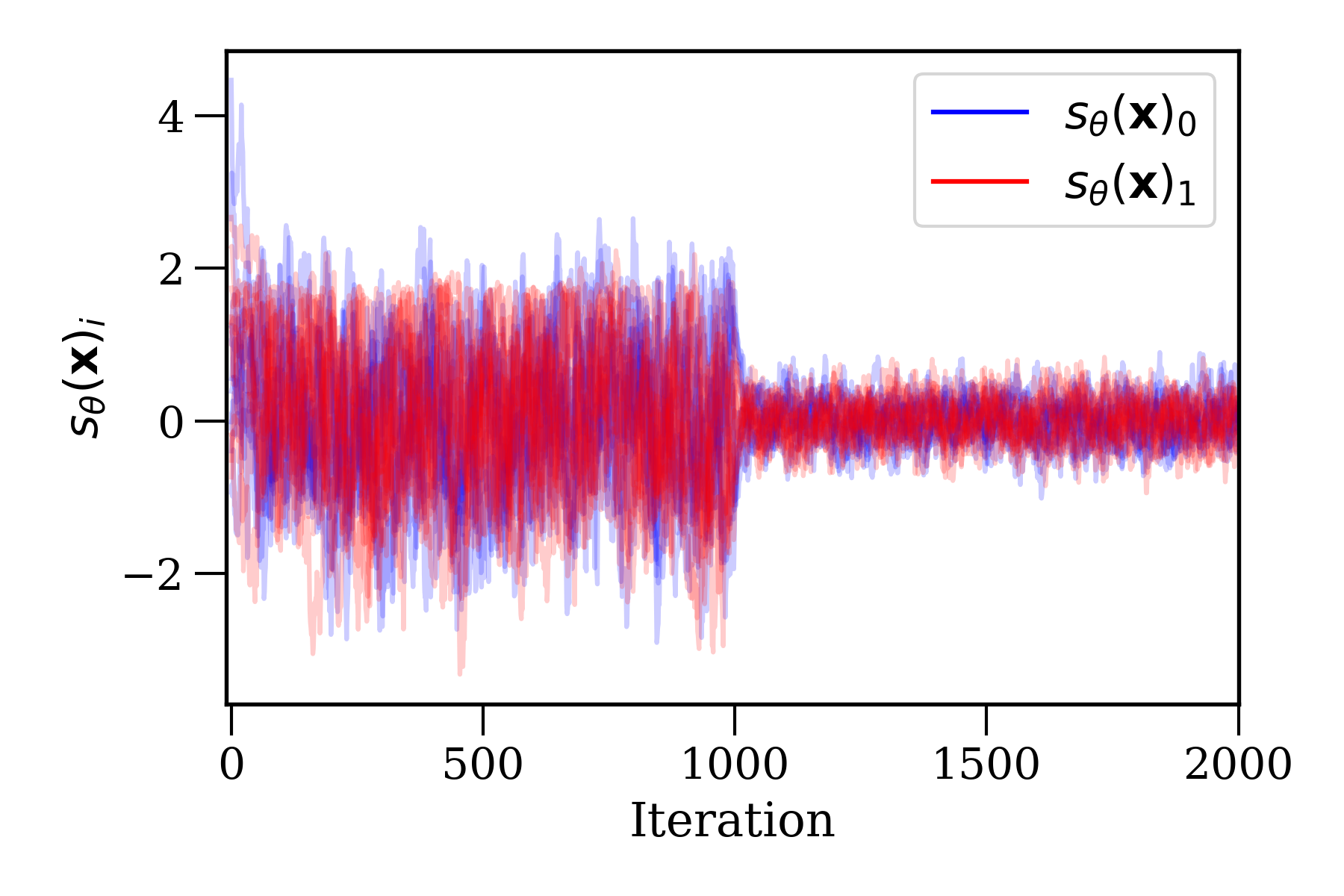}}%
    \subfigure[]{\label{fig:dsm_ald_10x_4}\includegraphics[width=0.45\linewidth]{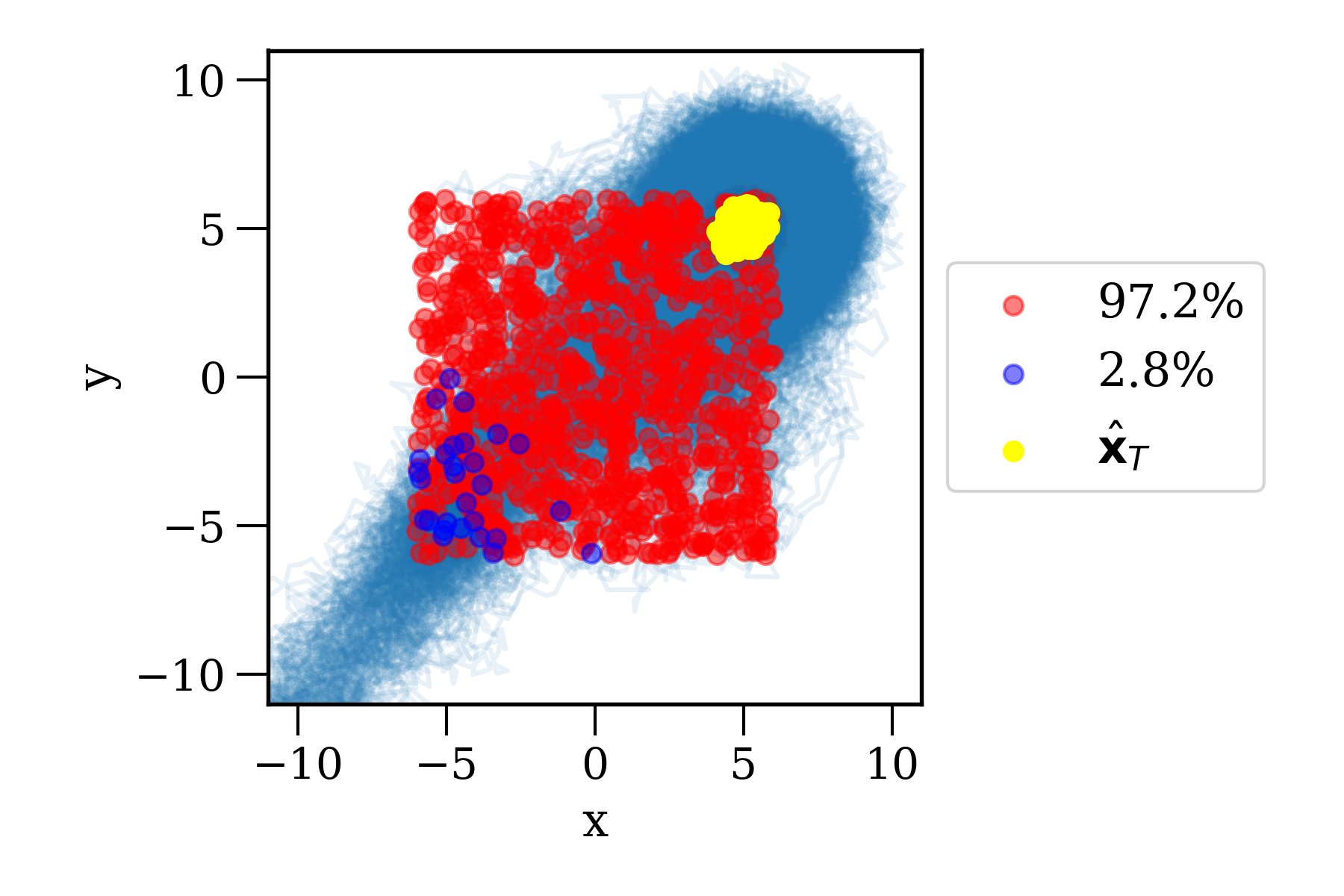}}}
    \caption{DSM with ALD mode collapse. The setup and figures are identical to Figure~\ref{fig:dsm ald example 2}, except that $p(\V{x})$ samples now have an imbalance of 10:1 between modes 1 (upper right) and 2 (lower left) respectively.}
    \label{fig:dsm ald 10x example 2}
\end{figure}






\clearpage

\subsection{High-dimensional class imbalance continued}\label{sec:imbalanced_experiments}

Figure~\ref{fig:sbm mnist 1/8 b=2} and \ref{fig:sbm mnist 1/8 b=1} present model generation results for a simplified version of the MNIST dataset. The data is limited to contain only the classes 1 and 8, which were chosen for their contrast in pixel space. The goal of Figure~\ref{fig:sbm mnist 1/8 b=2} was to demonstrate how Gaussian DSM SBMs perform poorly with ALD in the presence of asymmetric class representation, by inducing an imbalance between classes 1 and 8. However, mode collapse in Figure~\ref{fig:mnist18_b-2.0_image_grid_1000} occurred even without manipulating the class imbalance. Gaussian DSM suffering such issues in this minimal setting appears to contradict \citet{song2019generative}, where the motivation for combining DL and ALD was to overcome uneven mode weights. Moreover, it brings into question the cause of recent impressive generative results with SBMs, which may require the regularisation of many classes in the data to produce more general score estimates.
\begin{figure}[h]
    \centering
    \subfigure[Class ratio 1:1.]{\label{fig:mnist18_b-2.0_image_grid_1000}\includegraphics[width=0.25\linewidth]{Chapter4/Figs/mnist18_b-2.0_image_grid_1000.png}}
    \subfigure[Class ratio 2:1.]{\label{fig:mnist18_b-2.0_f-0.5_image_grid_1000}\includegraphics[width=0.25\linewidth]{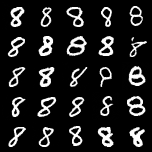}}
    \caption{Generated digits for a Gaussian DSM model trained for 20,000 steps on digits 1 and 8 from the MNIST dataset and sampled with 100 steps per level (s/l) of Gaussian ALD. No digits resembling a 1 are present for 25 samples, indicating a sampling process which induces (\textbf{a}) or exacerbates (\textbf{b}) the class imbalance.}
    \label{fig:sbm mnist 1/8 b=2}
\end{figure}

Figure~\ref{fig:sbm mnist 1/8 b=1} demonstrates intensive sampling results for HTDSM. Generated digits for a HTDSM model are trained in-line with Figure~\ref{fig:imbalanced} and sampled with varied diffusion type and steps per level (s/l). In Figure~\ref{fig:mnist18_b-1.0}, speckle is observed, whereas more sampling steps in Figure~\ref{fig:mnist18_b-1.0-10x again} leads to a more even class balance. The difference between Figures~\ref{fig:mnist18_b-1.0_sub-diff} and \ref{fig:mnist18_b-1.0_sub-diff-10x} confirms that sub-Gaussian diffusion can be used in high dimensions successfully, as long as the number of sampling steps is increased.
\begin{figure}[h]
    \centering
    \subfigure[ALD, 100 s/l.]{\label{fig:mnist18_b-1.0}\includegraphics[width=0.24\linewidth]{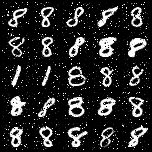}}
    \subfigure[ALD, 1,000 s/l.]{\label{fig:mnist18_b-1.0-10x again}\includegraphics[width=0.24\linewidth]{Chapter4/Figs/mnist18_b-1.0-10x.png}}
    \subfigure[Laplace ALD, 100 s/l.]{\label{fig:mnist18_b-1.0_sub-diff}\includegraphics[width=0.24\linewidth]{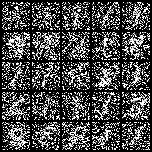}}
    \subfigure[Laplace ALD, 1,000 s/l.]{\label{fig:mnist18_b-1.0_sub-diff-10x}\includegraphics[width=0.24\linewidth]{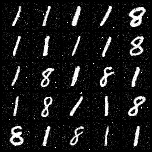}}
    \caption{Samples for a HTDSM model sampled with varied diffusion type and steps per level (s/l).}
    \label{fig:sbm mnist 1/8 b=1}
\end{figure}

\clearpage

\subsection{Full tabular results}\label{sec:tab_results}

\begin{table*}[ht]
    \centering
    \caption{All DGM metrics for unconditional samples from a model trained with HTDSM, and sampled from using Gaussian ALD, for different values of $\beta$ on the MNIST dataset. Arrows indicate whether higher ($\uparrow$) or lower ($\downarrow$) metric values are better.}
    \begin{tabular}{lrrrr}
        \toprule
         & $\beta=1.0$ & $\beta=1.5$ & $\beta=2.0$ & $\beta=2.5$ \\
        \midrule
        Precision $\uparrow$ & \textbf{0.9417} & 0.9244 & 0.912 & 0.894 \\
        Recall $\uparrow$ & 0.8634 & 0.9023 & \textbf{0.936} & 0.905 \\
        Density $\uparrow$ & \textbf{0.9869} & 0.9210 & 0.867 & 0.849 \\
        Coverage $\uparrow$ & \textbf{0.9112} & 0.7816 & 0.780 & 0.733 \\
        \midrule
        IS $\uparrow$ & $2.020\pm0.018$ & $2.084\pm0.038$ & $2.037\pm0.037$ & $\V{2.109\pm0.032}$ \\
        KID $\downarrow$ & $0.075\pm0.002$ & $0.020\pm0.001$ & $0.016\pm0.002$ & $\V{0.008\pm0.001}$ \\
        \bottomrule
    \end{tabular}
    \label{tab:sbm mnist all}
\end{table*}

\begin{table*}[ht]
    \centering
    \caption{All DGM metrics for unconditional samples from a model trained with HTDSM, and sampled from using Gaussian ALD, for different values of $\beta$ on the Fashion-MNIST dataset.}
    \begin{tabular}{lrrrr}
        \toprule
         & $\beta=1.0$ & $\beta=1.5$ & $\beta=2.0$ & $\beta=2.5$ \\
        \midrule
        Precision $\uparrow$ & 0.1244 & 0.922 & 0.884 & \textbf{0.9230} \\
        Recall $\uparrow$ & \textbf{0.9592} & 0.765 & 0.787 & 0.7754 \\
        Density $\uparrow$ & 0.0355 & 1.541 & 1.406 & \textbf{1.600} \\
        Coverage $\uparrow$ & 0.0351 & \textbf{0.651} & 0.587 & 0.5962 \\
        \midrule
        IS $\uparrow$ & $3.120\pm0.084$ & $\V{3.790\pm0.090}$ & $3.646\pm0.109$ & $3.591\pm0.120$ \\
        KID $\downarrow$ & $0.146\pm0.002$ & $\V{0.020\pm0.001}$ & $0.023\pm0.002$ & $0.027\pm0.002$ \\
        \bottomrule
    \end{tabular}
    \label{tab:sbm fashion all}
\end{table*}

\clearpage

\subsection{Samples}\label{sec:samples}

\begin{figure}[ht]
    \centering
    \subfigure[$\beta=1.5$.]{\label{fig:mnist_b-1.5_image_grid_160000}\includegraphics[width=0.25\linewidth]{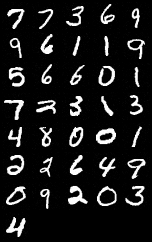}} \hfill%
    \subfigure[$\beta=2.0$.]{\label{fig:mnist_b-2.0_image_grid_200000}\includegraphics[width=0.25\linewidth]{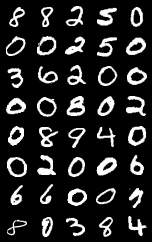}} \hfill%
    \subfigure[$\beta=2.5$]{\label{fig:mnist_b-2.5_image_grid_200000}\includegraphics[width=0.25\linewidth]{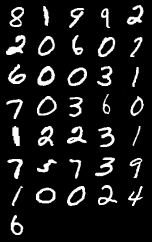}}
    \caption{Unconditional samples from a model trained with HTDSM, and sampled from using Gaussian ALD, for different values of $\beta$ on the MNIST dataset.}
    \label{fig:sbm mnist unconditional}
\end{figure}

\begin{figure}[h!]
    \centering
    \subfigure[$\beta=1.5$.]{\label{fig:fashion_b-1.5_image_grid_200000}\includegraphics[width=0.25\linewidth]{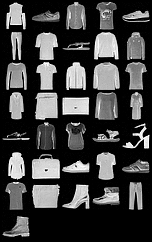}} \hfill%
    \subfigure[$\beta=2.0$.]{\label{fig:fashion_b-2.0_image_grid_200000}\includegraphics[width=0.25\linewidth]{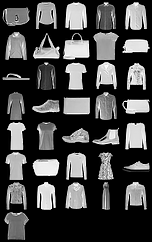}} \hfill%
    \subfigure[$\beta=2.5$]{\label{fig:fashion_b-2.5_image_grid_200000}\includegraphics[width=0.25\linewidth]{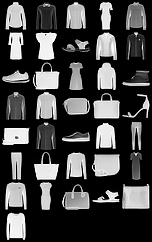}}
    \caption{Unconditional samples from a model trained with HTDSM, and sampled from using Gaussian ALD, for different values of $\beta$ on the Fashion-MNIST dataset.}
    \label{fig:sbm fashion unconditional}
\end{figure}

\begin{figure}[h!]
    \centering
    \subfigure[$\beta=1.5$.]{\label{fig:cifar_b-1.5_image_grid_200000}\includegraphics[width=0.3\linewidth]{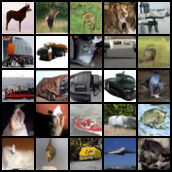}} \hfill%
    \subfigure[$\beta=2.0$.]{\label{fig:cifar_b-2.0_image_grid_200000}\includegraphics[width=0.3\linewidth]{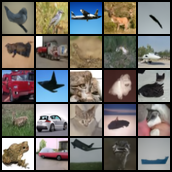}} \hfill%
    \subfigure[$\beta=2.5$]{\label{fig:cifar_b-2.5_image_grid_200000}\includegraphics[width=0.3\linewidth]{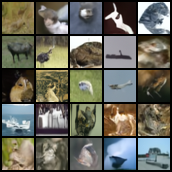}}
    \caption{Unconditional samples from a model trained with HTDSM, and sampled from using Gaussian ALD, for different values of $\beta$ on the Cifar-10 dataset.}
    \label{fig:sbm fashion unconditional}
\end{figure}

\begin{figure}[h!]
    \centering
    \subfigure[$\beta=1.5$.]{\label{fig:celeba_b-1.5_image_grid_200000}\includegraphics[width=0.3\linewidth]{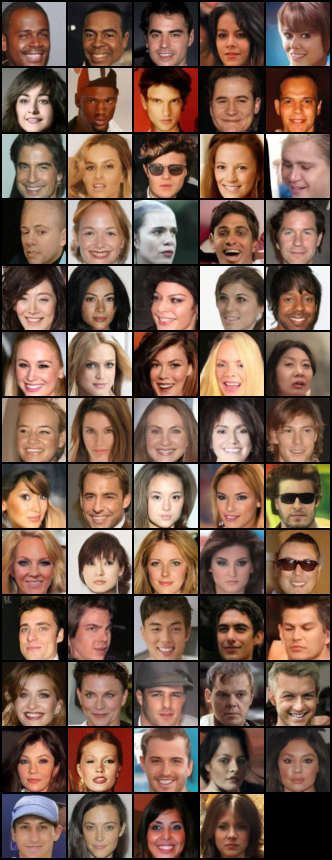}} \hfill%
    \subfigure[$\beta=2.0$.]{\label{fig:celeba_b-2.0_image_grid_200000}\includegraphics[width=0.3\linewidth]{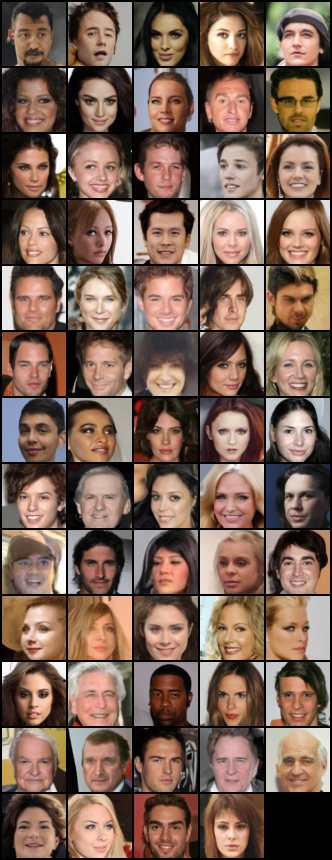}} \hfill%
    \subfigure[$\beta=2.5$]{\label{fig:celeba_b-2.5_image_grid_200000}\includegraphics[width=0.3\linewidth]{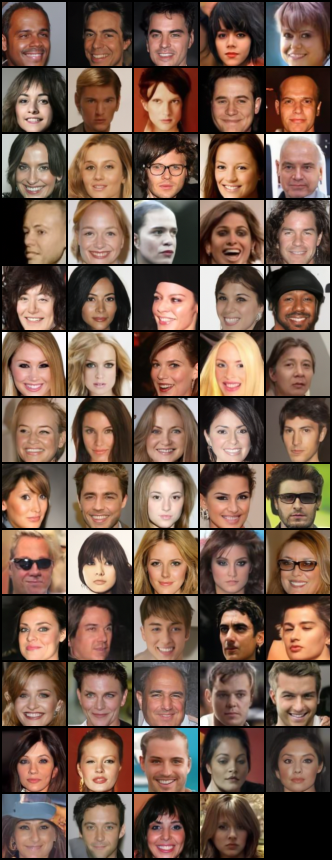}}
    \caption{Unconditional samples from a model trained with HTDSM, and sampled from using Gaussian ALD, for different values of $\beta$ on the CelebA dataset.}
    \label{fig:sbm fashion unconditional}
\end{figure}

\end{document}